\documentclass[conference]{IEEEtran}
\usepackage{cite}
\usepackage{amsmath,amssymb,amsfonts}
\usepackage{algorithmic}
\usepackage{graphicx}
\usepackage{textcomp}
\usepackage{xcolor}
\usepackage{algorithm}
\usepackage{url}
\usepackage{bbding}
\usepackage{amsthm}
\usepackage{color}
\usepackage{framed}
\usepackage{verbatim}
\usepackage[caption=false, font=footnotesize]{subfig}
\usepackage{enumitem}

\newtheorem{theorem}{Theorem}
\newtheorem{definition}{Definition}
\newtheorem{lemma}{Lemma}
\newtheorem{proposition}{Proposition}

\def\BibTeX{{\rm B\kern-.05em{\sc i\kern-.025em b}\kern-.08em
    T\kern-.1667em\lower.7ex\hbox{E}\kern-.125emX}}
\begin{document}

\title{FedTilt: Towards Multi-Level Fairness-Preserving and Robust Federated Learning}

\author{\IEEEauthorblockN{Binghui Zhang}
\IEEEauthorblockA{Illinois Institue of Technology}
\IEEEauthorblockA{bzhang57@hawk.iit.edu}
\and \IEEEauthorblockN{Luis Mares De La Cruz}
\IEEEauthorblockA{Illinois Institue of Technology}
\IEEEauthorblockA{lmaresdelacruz@hawk.iit.edu}
\and \IEEEauthorblockN{Binghui Wang}
\IEEEauthorblockA{Illinois Institue of Technology}
\IEEEauthorblockA{bwang70@iit.edu}
}

\maketitle

\begin{abstract}

Federated Learning (FL) is an emerging decentralized learning paradigm that can 
partly address the privacy concern that cannot be handled by traditional centralized and distributed learning. 
Further, to make FL practical, it is also necessary
to consider constraints such as 
fairness and robustness. 
However, existing robust FL methods often produce unfair models, and existing fair FL methods only consider one-level (client) fairness and are not robust to persistent outliers (i.e., injected outliers into each training round) that are common in real-world FL settings. 
   We propose \texttt{FedTilt}, a novel FL that can preserve multi-level fairness and be robust to outliers.  
    In particular, we consider two common levels of fairness, i.e., \emph{client fairness}---uniformity of performance across clients, 
    and \emph{client data fairness}---uniformity of performance across different classes of data within a client. 
    \texttt{FedTilt} is inspired by the recently proposed tilted empirical risk minimization, 
    which introduces tilt hyperparameters that can be flexibly tuned.
    Theoretically, we show how tuning tilt values can achieve the two-level fairness and mitigate the persistent outliers, and derive the convergence condition of \texttt{FedTilt} as well. 
    Empirically, our evaluation results on a suite of realistic federated datasets  in diverse settings show the effectiveness and flexibility of the  \texttt{FedTilt} framework and  the superiority to the state-of-the-arts.
\end{abstract}


\section{Introduction}
\label{sec:intro}

Federated Learning (FL)~\cite{mcmahan2017communication}
is an emerging decentralized learning paradigm that 
enables a server and clients to perform joint learning  
without any data sharing, which partly addresses the privacy concern that could not be handled by traditional centralized and distributed learning. 
To make FL practical, it is necessary for the deployed system
to also consider the reasonable constraints such as \emph{fairness} and \emph{robustness}. 
The reasons are as follows: in order to incentivize more clients 
to participate FL, it would be better for all clients to obtain similar performance. 
Moreover, clients often have "outlier" data, e.g., data with large noises or corruptions. Using these data for training could yield negative affect the FL model performance, and hence FL should be robust to these outliers.    

Most of the existing works consider fairness or robustness in FL \emph{separately}, as satisfying
both constraints is challenging~\cite{kairouz2021advances,li2021ditto}. For instance, these methods~\cite{mohri2019agnostic,deng2020distributionally,li2020lotteryfl,li2020fair,ma2022layer} achieve the fairness goal, while robust methods majorly use robust aggregation~\cite{blanchard2017machine,chen2017distributed,guerraoui2018hidden,yin2018byzantine,guerraoui2018hidden,chen2018draco,pillutla2019robust,wu2020federated,xu2022fedcorr,karimireddy2021learning,farhadkhani2022byzantine,zhang2022fldetector,cao2023fedrecover,yang2024breaking,yang2024distributed}. 
To our best knowledge, Ditto~\cite{li2021ditto} is the only method that accounts for both fairness and robustness. However, Ditto has the below drawbacks: 
1) Ditto and all fair FL methods only consider one-level \emph{client fairness}, i.e., they require testing data across all clients  achieve close performance. 
We advocate that, besides the client fairness, 
the performance of data from different groups (or classes) within a client should be also similar (we term \emph{client data fairness}), which also aligns with the fairness definition (e.g., disparate treatment) in centralized learning~\cite{zafar2017fairness}.   
However, directly applying the existing methods cannot achieve promising multi-level fairness performance. 
2) We show that Ditto (see Table~\ref{competitiveness_pixels}) 
is not robust to \emph{persistent} outliers (e.g., large corrupted data), where the outliers are injected into participating clients' data in \emph{all} communication rounds, 
instead of \emph{only once} during FL training. 
We note that such a scenario is more practical, as different clients are often selected to participate in training in different communication rounds.
The goal of this paper is to achieve the multi-level fairness, as well as the robustness to persistent outliers 
in FL. To this end, we design a flexible FL method dubbed \texttt{FedTilt}. 
We also show {existing fair FL methods (e.g., FedAvg and Ditto)} are special cases of FedTilt. 
Further, we derive the convergence condition of FedTilt.
We finally evaluate FedTilt and compare it with the state-of-the-art fair FL methods on multiple datasets. Our results show FedTilt obtains a comparable/better clean testing accuracy, 
and achieves better two-level fairness and better robustness to persistent outliers.

\section{Background}
\label{sec:background}

\textbf{Federated learning (FL).} 
Suppose a total of $N$ clients $\{C_n\}_{n \in [N]}$ participate in FL, where each client $C_n$ owns data $z_n = (x_n, y_n)$ from a distribution $\mathcal{D}_n$, where $x_n$ is the feature vector and $y_n$ is the label.  Traditionally, FL considers a shared global (server) model for all clients and optimizes the (local/global) objectives as follows: 
\begin{align}
& \textbf{Global obj.: } {\bf w} = \arg \min_{\bf w} G({\bf w}, \{{\bf w}_n\}), 
\label{eqn:FL_globalobj} \\
& \textbf{Local obj.: } 
{\bf w}_n = \arg \min_{{\bf w}_n} F_n({\bf w}_n; {\bf w}), 
\end{align}
where $F_n({\bf w}_n; {\bf w})=\mathbb{E}_{z_n \sim \mathcal{D}_n} (l(z_n; {\bf w}))$ is the average local loss over $C_n$'s data; and $l(\cdot;\cdot)$ is a user-specified loss function. 
${\bf w}_n$ is the client model on $C_n$. 
$G(\cdot)$ is a global aggregation function. 
For instance, the well-known 
FedAvg~\cite{mcmahan2017communication} 
uses an average aggregation to update the global model, i.e., ${\bf w} = \frac{1}{N} \sum_{n\in[N]} {\bf w}_n$.
Specifically, FL with FedAvg is trained 
as below: 1) Server initializes a global model ${\bf w}$ and sends it to all clients; 2) Each client $C_n$ minimizes 
$F_n({\bf w}_n; {\bf w})$ to obtain a client model ${\bf w}_n$, e.g., ${\bf w}_n = {\bf w} - \eta \nabla_{{\bf w}_n} F_n({\bf w}_n; {\bf w})$, 
and sends ${\bf w}_n$ 
to the server; 3) Server updates the global model ${\bf w}$ by averaging the received client models ${\bf w}_n$, 
and broadcasts the updated ${\bf w}$ to all clients. Such steps are  performed iteratively until convergence or reaching maximum global rounds. 

\vspace{+0.05in}
 \noindent 
 \textbf{TERM} Our method is inspired by the recently proposed tilted empirical risk minimization (TERM)~\cite{li2021tilted} for centralized learning.  
ERM has been used in almost all the  existing centralized and distributed learning objectives. 
However, recent studies~\cite{jiang2018mentornet,khetan2018learning,hashimoto2018fairness,samadi2018price} show ERM performs poorly when  average performance is not an appropriate surrogate for the problem of interest, e.g., 
learning in the presence of outliers (e.g., noisy, corrupted, or mislabeled data) and
ensuring the fairness for subgroups within a population,
which commonly exist in real applications. 
TERM is a recent framework aiming to address these problems 
for centralized learning. 
 Specifically, TERM generalizes ERM by introducing a hyperparameter called \emph{tilt}. Given an average loss $R({\bf w})=\mathbb{E}_z[l(z;{\bf w})]$ in ERM, the corresponding $t$-tilted loss in TERM is defined as:
 \begin{align}
     \label{eqn:TERM}
     \tilde{R}(t;{\bf w}) = 
     {1}/{t} \cdot \log(\mathbb{E}_z[e^{t\cdot l(z;{\bf w})}]).
 \end{align}
TERM is flexible via tuning $t$: 
1) It recovers ERM (average-loss) with t=0 (i.e., $\tilde{R}(0;{\bf w})$=$R({\bf w})$); the max-loss $\tilde{R}(+\infty;{\bf w})$=$\max_{i} l(z_i; {\bf w})$ with t$\rightarrow$$+\infty$; and the min-loss $\tilde{R}(-\infty;{\bf w})$=$\min_{i} l(z_i; {\bf w})$ with t$\rightarrow$$-\infty$; 
2) For t$>$0, 
it enables a smooth tradeoff between the average-loss and max-loss. 
TERM can selectively improve the worst losses by penalizing the average
performance, thus promoting uniformity or fairness. 
3) For t$<$0, the solutions achieve a smooth tradeoff between average-loss and min-loss, which can focus on relatively small losses, ignoring large losses caused by outliers. 
\section{FedTilt} \label{sec:fedtilt}
We design FedTilt to achieve multi-level fairness and robustness to persistent outliers. 
We show FedAvg and recent fair FL methods such as FedProx~\cite{li2020federated} and Ditto~\cite{li2021ditto} are special cases of FedTilt. We also derive convergence results of FedTilt. 

\subsection{Background on FedProx and Ditto}
\label{app:background}
\noindent \textbf{FedProx~\cite{li2020federatedoptimization}.}
In practice, the data distribution across clients can differ. To account for such data heterogeneity that often leads to unfair performance, FedProx proposes a proximal term to the local objective. 
Each client $C_n$ minimizes the local objective as below to learn the shared global model ${\bf w}$:
\begin{align}
& \textbf{Global obj.: } {\bf w} = \arg \min_{\bf w} G({\bf w}, \{{\bf w}_n\}), \label{obj:fedprox} \\ 
   &  \textbf{Local obj.: }  {\bf w}_n = \arg\min_{{\bf w}_n}
    L_n({\bf w}_n, {\bf w}) \nonumber \\
    & \qquad \qquad \qquad \, \, = F_n({\bf w}) + 
    \frac{\mu}{2} \|{\bf w}_n - {\bf w} \|^2, \label{localobj:fedprox} 
\end{align}
where the hyperparameter $\mu$ tradeoffs the local objective and the proximal term 
$\|{\bf w}_n - {\bf w} \|^2$, 
which aims to restrict
the intermediate local models ${\bf w}_n$ in each client to be closer to the global model ${\bf w}$, thus mitigating unfairness. The proximal term also shows to improve the stability of training.  
Note that when $\mu=0$, FedProx reduces to the FedAvg.

\vspace{+0.05in}
\noindent \textbf{Ditto~\cite{li2021ditto}.}
The  state-of-the-art Ditto differs 
from other FL methods (e.g., FedAvg and FedProx~\cite{li2020federatedoptimization}) 
by learning personalized client models via federated multi-task learning. Specifically, Ditto considers optimizing both the global objective and local objective and simultaneously learns the global model and a local model (i.e., ${\bf v}_n$) per client $C_n$ as below:
\begin{align}
     & \textbf{Global obj.: } {\bf w}^* \in \arg\min\nolimits_{\bf w} G({\bf w}, \{{\bf w}_n\}),  \label{obj:ditto_global} \\ 
    & 
\textbf{Local obj.: }  
    {\bf v}_n^* = \arg
    \min_{{\bf v}_n}L_n({\bf v}_n, {\bf w}^*) \nonumber \\
    & \qquad \qquad \qquad \, \, = F_n({\bf v}_n) + 
    \frac{\mu}{2} \|{\bf v}_n - {\bf w}^* \|^2 \label{obj:ditto_local} 
\end{align}
where it uses the average aggregation in $G(\cdot)$ by default and the hyperparameter $\mu$ tradeoffs the local client loss and the closeness between personalized client models and global models (which ensures client fairness). For instance, when $\mu=0$, Ditto reduces to training local client models $\{{\bf v}_n\}$; and when $\mu=+\infty$, all client models degenerate to the global model ${\bf w}$, making Ditto recover the FedAvg. Hence, through $\mu$, Ditto can achieve a promising fairness across clients, and maintain the FL performance as well. 

\subsection{Problem definition and design goals} 
We focus on multi-level fairness in FL, 
particularly both the client fairness and client data fairness\footnote{The fairness definitions in the paper follow 
existing fair FL methods~\cite{li2021ditto,li2020fair}, which are 
somewhat different from those 
in algorithmic fairness such as predictive equality, conditional statistical parity~\cite{dwork2012fairness,hardt2016equality,corbett2017algorithmic}.}.  
\begin{definition}[Client fairness]
\label{def:clinetfairness}
We say a global model ${\bf w}^a$ is more fair than 
another global model 
${\bf w}^b$ with respect to all clients $\{C_n\}_{n \in [N]}$, if all clients' performance are closer to each other when using ${\bf w}^a$ than using ${\bf w}^b$. 
\end{definition}

\begin{definition}[Client data fairness]
\label{def:clinetdatafairness}
A client $C_n$'s model is more fair than ${\bf w}_n^b$ with respect to a k-class data if the performance of ${\bf w}_n^a$ on all k classes is more uniform than ${\bf w}_n^b$.
\end{definition}

Client fairness requires different clients have close performance, while client data fairness further requires data from different classes also have close performance. Our goal is to design a framework that can achieve the above two-level fairness\footnote{It can be easily generalized to more-level fairness due to its flexibility.}, as well as  be robust to persistent outliers (e.g., injected corrupted data or large noisy data in every training round).
Our main idea is leveraging the TERM framework~\cite{li2021tilted}.

\subsection{\texttt{\textbf{FedTilt}} objective}
{FedTilt} introduces both a global objective and a local objective that aims to learn 
a global model ${\bf w}$ and a \emph{personalized} local model ${\bf v}_n$ per client, respectively. 
The general form of the FedTilt objective function is defined as follows: 
\begin{align}
 & \textbf{Global obj.: }   
     {\bf w}^* \in \arg \min_{{\bf w}} G({\bf w}, \{{\bf w}_n\});  \label{eqn:fedtilt_globalobj} \\ 
   & \textbf{Local obj.: }  \min_{{\bf v}_n} L_n({\bf v}_n, {\bf w}^*). \label{eqn:fedtilt_localobj} 
\end{align}
The global model ${\bf w}$ is updated via client models $\{{\bf w}_n\}$, and a local loss $L_n$ is defined per client $C_n$. The above problem is a bi-level optimization problem, where obtaining personalized client models $\{{\bf v}_n\}$ needs the optimal global model ${\bf w}^*$. We instantiate $G$ and $L_n$ via customized tilted loss to achieve client and client data fairness and robustness.

\noindent {\bf Achieving client fairness: 
tilted loss for the global objective.} 
Via Def.~\ref{def:clinetfairness}, client fairness is achieved and performance are similar when data are homogeneous across clients and we ensure all client models to be close to the global model. 
We define the tilted loss for the global objective as:

{\vspace{-4mm}
\footnotesize
\begin{align}
     G({\bf w}, \{{\bf w}_n\})  & = \tilde{R}_G(q; \{{\bf w}_n\}, {\bf w}) = \frac{1}{q} \log \big( \frac{1}{N} \sum_{n \in [N]} e^{q \cdot \texttt{dist}({\bf w}_n, {\bf w})}\big) 
     \label{GlobalObj} 
\end{align}
}

\noindent \emph{\textbf{Properties of the tilted global loss:}} 
When $q \rightarrow +\infty$, 
$\tilde{R}_G (+\infty; \{{\bf w}_n\}, {\bf w}) \rightarrow \max_{n} \texttt{dist}({\bf w}_n, {\bf w})$. Minimizing this max loss 
makes all client models 
$\{{\bf w}_n\}$ close to the global model ${\bf w}$, thus ensuring client fairness. 
On the other hand, 
{when $q \rightarrow -\infty$, 
$\tilde{R}_G (-\infty; \{{\bf w}_n\}, {\bf w}) \rightarrow \min_{n} \texttt{dist}({\bf w}_n, {\bf w})$. Minimizing this min loss focuses on the clients with small loss, thus defending against 
clients whose local losses are high (e.g., caused by outlier data).
}
When setting $q=0$ and $\texttt{dist}({\bf w}_n, {\bf w}) = ||{\bf w}_n - {\bf w}||_2^2$, 
$\tilde{R}_G(0; \{{\bf w}_n\}, {\bf w}) = \frac{1}{N}\sum_{n\in[N]}||{\bf w}_n - {\bf w}||_2^2$. 
Minimizing tilted global loss recovers the average aggregation, which is same as FedAvg ~\cite{mcmahan2017communication}. 

\noindent {\bf Achieving client data fairness and robustness to outliers: two-level tilted loss for the local objective.} 
The local objective 
aims to quantify the wellness of each personalized 
client model w.r.t. the associated client data. 
If we ensure data from different classes have close performance, the client data fairness is achieved. Moreover, if the local model is not affected by the outliers in client's data, it is robust to the outliers. 
We design the below local objective  which includes a two-level tilted loss and a regularization term (inspired by Ditto~\cite{li2021ditto}) to achieve both goals. \footnote{$\tilde{R}_n (\tau, \lambda; {\bf v}_n) := \frac{1}{\tau} \log \Big(\frac{1}{|D_n|}  \sum_{D_n^k \in [D_n]} |D_n^k| e^{\tau \cdot \tilde{R}_n^k(\lambda;{\bf v}_n)} \Big),\\ \tilde{R}_n^k(\lambda;{\bf v}_n) := 
    \frac{1}{\lambda} \log \Big( \frac{1}{|D_n^k|} \sum_{z \in {D_n^k}} e^{\lambda \cdot l(z; {\bf v}_n)} \Big)$\\ where $D_n^k$ represents the data in the client $C_n$ belonging to class $k$ and $D_n = \{D_n^k\}_{k=1}^K$ 
includes data from all classes. 
$\tilde{R}_n (\tau, \lambda; {\bf v}_n)$ is 
$C_n$'s tilted loss and $\tilde{R}_n^k(\lambda;{\bf v}_n)$ is the tilted loss for class-$k$  data in $C_n$. }
\begin{align} 
      L_n({\bf v}_n, {\bf w}) = \tilde{R}_n (\tau, \lambda; {\bf v}_n) + \frac{\mu}{2} \|{\bf v}_n -{\bf w} \|^2, \label{eqn:local_fedtile_reg} 
\end{align}

\noindent \emph{\textbf{Properties of the tilted local loss 
}}  
1)  When $\tau \rightarrow +\infty$, 
$\tilde{R}_n (+\infty, \lambda; {\bf v}_n) \rightarrow \max_{D_n^k} \tilde{R}_n^k(\lambda;{\bf v}_n)$. Minimizing this max loss can promote uniformity of different classes' data in client $C_n$, thus ensuring client data fairness. 
2) When $\tau \rightarrow -\infty$, 
$\tilde{R}_n (+\infty, \lambda; {\bf v}_n) \rightarrow \min_{D_n^k} \tilde{R}_n^k(\lambda;{\bf v}_n)$. Minimizing this min loss indicates only focusing on the class $k$ whose overall data loss is the smallest can mitigate outliers from other classes. 
3) When $\lambda \rightarrow +\infty$, $\tilde{R}_n^k(+\infty;{\bf v}_n) \rightarrow \max_{z \in D_n^k} l(z;{\bf v}_n)$. Minimizing this max loss means  promoting uniformity of all data from the class $k$. With $\tau \rightarrow +\infty$, the client data fairness is further enhanced. 
4) When $\lambda \rightarrow -\infty$, 
$\tilde{R}_n^k(-\infty;{\bf v}_n) \rightarrow \min_{z \in D_n^k} l(z; {\bf v}_n)$. 
Minimizing this min loss indicates only focusing on the data from class-$k$ with the smallest loss, thus can mitigate all the outliers  existed in the class-$k$ data. 
5) When $\tau=\lambda$, $\tilde{R}_n (\tau, \tau; {\bf v}_n) \rightarrow \frac{1}{\tau} \log \big( \frac{1}{|D_n|} \sum_{z \in C_n} e^{\tau \cdot l(z; {\bf v}_n)} \big)$, which reduces to the one-level TERM; 
6) When $\tau \rightarrow 0$ and $\lambda \rightarrow 0$, $\tilde{R}_n (0, 0; {\bf v}_n) \rightarrow \frac{1}{|D_n|} \sum_{z \in C_n} l(z; {\bf v}_n)$, which reduces to the classic loss used in Eqn~\ref{eqn:FL_globalobj}.

\begin{table}[!t]
\footnotesize
\caption{Effect of tilt hyperparameters.} 
\centering
\addtolength{\tabcolsep}{-2pt}
\begin{minipage}{.3\linewidth}
\begin{tabular}{ll}
 \multicolumn{1}{l}{$q$} & \multicolumn{1}{l}{Client fair.} \\
    \hline
       $q>0$ & {High}  \\
      $q=0$ &  {Medium} \\
      $q<0$ &  Low  \\
     \hline 
\end{tabular}
\end{minipage}
\hfill
\begin{minipage}{.65\linewidth}
    \begin{tabular}{llll}
  \multicolumn{1}{l}{$\tau$} & \multicolumn{1}{l}{$\lambda$} &  \multicolumn{1}{l}{Client data fair.} & {Rob.}\\
    \hline
    $\tau >0$ &  $\lambda >0$  &    Very High  & Low \\
    $\tau >0$  &  $\lambda <0  $  &    High  &  High \\
    $\tau <0$  &  $\lambda >0$  &    High  &  High \\
    $\tau =0$ &  $\lambda =0$  &    Medium  &  Medium \\
    $\tau <0$ &  $\lambda <0$  &    Low  &  Very High \\
    \hline
\end{tabular}
\end{minipage}
\label{tbl:tiltsummary}
\vspace{-0.25in}
\end{table}

\noindent {\bf Remark.} Theoretically, FedTilt achieves a two-level fairness and robustness tradeoff, by flexibly tuning the tilt hyperparameters in the global and local objectives. 
In other words, it is impossible to obtain the optimal two-level fairness and robustness simultaneously. This tradeoff is also reflected in Table~\ref{tbl:tiltsummary}. 
For instance, (more) positive $q$ yields (more) client fairness, and (more) positive $\tau$ and (more) negative $\lambda$ yields (more) client data fairness, but (less) robustness. Practically, these properties guide us to set the proper values of $q$, $\tau$, and $\lambda$ to obtain a promising tradeoff in our experiments.

\vspace{-2mm}
\subsection{\texttt{\textbf{FedTilt}} Solver}
\vspace{-1mm}
Solving FedTilt requires updates on all clients and the server via multiple global communication rounds and local epochs. 
We propose to alternatively solve for the global model ${\bf w}^*$ and personalized client models $\{{\bf v}_n^*]\}_{n\in[N]}$, which is summarized in Algorithm 1.
Specifically, with an initialized global model ${\bf w}^0$ and personalized client models $\{{\bf v}_n^0\}_{n \in N}$ (Line 1), the optimization is performed in two iterative steps (Line 2-Line 9): (1) each personalized client  model $\{{\bf v}_n^t\}$ is trained locally on per client's data $C_n$ by minimizing the local objective $L_n({\bf v}_n^{t-1}; {\bf w}^{t-1})$ with the current global model ${\bf w}^{t-1}$ and 
${\bf v}_n^{t-1}$ (Line 11-Line 19); and (2) global model ${\bf w}^t$ is then updated on the server via minimizing the global objective $\tilde{R}_G(q;\{{\bf w}_n^t\}, {\bf w}^{t-1})$, which leverages  clients' intermediate  models $\{{\bf w}_n^t\}$ and the current global model  ${\bf w}^{t-1}$ (Line 20-Line 24). 
Note that the clients' intermediate models are updated via 
minimizing the client loss $\tilde{R}_n(\tau, \lambda; {\bf w}^{t-1})$.  

\vspace{-2mm}
\subsection{Theoretical results}

\subsubsection{Relation to other methods} 
We show {FedAvg~\cite{mcmahan2017communication}, FedProx~\cite{li2020federated}, and Ditto~\cite{li2021ditto} are special cases of FedTilt.} 

\begin{proposition} 
\label{thm:fedavg_special}
FedAvg is a special case of  FedTilt, i.e., when the tilt hyperparameters $q=0$, $\tau=0$, $\lambda=0$, $\mu=+\infty$, and \texttt{dist} is 
Euclidean.
\title{propfedavgspecial}
\end{proposition}

\begin{proposition}
\label{thm:fedprox_special}
FedProx is a special case of FedTilt, i.e., when 
$q=0$, $\tau=0$, $\lambda=0$,  
${\bf v}_n={\bf w}_n$, and \texttt{dist} is Euclidean.
\title{propfedproxspecial}
\end{proposition}

\begin{proposition}
\vspace{-2mm}
\label{thm:ditto_special}
\texttt{Ditto} is a special case of FedTilt, when 
$q=0$, $\tau=0$, $\lambda=0$, and \texttt{dist} is Euclidean. 
\label{propdittospecial}
\vspace{-2mm}
\end{proposition}

The proofs of propositions are included in the appendix \ref{app:proofprop}

\subsubsection{Convergence results}

Note that optimizing the global model ${\bf w}$
does not depend on any personalized client models $\{{\bf v}_n\}_{n\in[N]}$, but the model updates $\{{\bf w}_n\}_{n \in [N]}$. Hence, FedTilt has the same global convergence rate with the standard solver that we use for solving a convex $G$ that does not learn personalized client models. 

For instance, by setting $q = 0$ and the distance function $\texttt{dist}$ is the Euclidean distance, $G$ becomes the average aggregation (See Proposition~\ref{thm:fedavg_special}), and the global model converges at a rate of $O(1/t)$~\cite{li2019convergence}, with $t$ the global round index.
Under this observation, we present the local convergence result of client models via  Algorithm 1, where we assume the loss function $l$ is smooth and strongly convex, following the existing works~\cite{li2019convergence,li2021ditto,li2021tilted}, and the global model ${\bf w}^t$ converges to its optimal ${\bf w}^*$.

\begin{theorem}[Convergence results of client models with Algorithm~\ref{app:Algorithm} (Informal); formal statement and proof are shown in Appendix~\ref{app:provelocalcvg}]
\vspace{-2mm}
\label{thm:cvg}
Assume the loss function $l$ in the local objective is smooth and strongly convex. 
If the global model ${\bf w}^t$ converges to ${\bf w}^*$ with rate $g(t)$, then there exists a constant $C<+\infty$ such that for $\tau>0, \lambda>0$ and any $\mu$, and for $n \in [N]$, ${\bf v}_n^t$ converges to ${\bf v}_n^* := \arg\min L_n({\bf v}_n, {\bf w}^*)$ with rate $C g(t)$. 
\vspace{-1mm}
\end{theorem}

\subsubsection{Convergence Results of FedTilt}
\label{app:provelocalcvg}
We first introduce the following definitions, assumptions, and lemmas.
Then we proof the convergence conditions of FedTilt. 

The overall proof idea is as follows: 1) Assume that standard loss $l$ is convex and strongly smooth, a standard assumption used in most FL methods~\cite{li2021tilted,li2021ditto,li2020federated,li2019convergence};
2) Show the class-wise one-level $\lambda$-tilted loss $ \tilde{R}_{n,k}(\lambda;{\bf v}_n)$ is convex and smooth based on 1); 
3) Further show the two-level $(\tau, \lambda)$-tilted client loss $\tilde{R}_{n}(\tau, \lambda; {\bf v}_n)$ and local objective $L_n({\bf v}_n,  {\bf w})$ are convex and smooth based on 1) and 2);
4) Show the global loss is convergent based on Ditto~\cite{li2021ditto}.  
5) Finally, combining the convergence property of local objective and global objective, we show the convergence condition of FedTilt. 

 Definition of {\bf Smooth function, Strongly convex function, and PL inequality} are included in the appendix.

\noindent {\bf Assumption 1} (Smooth and strongly convex loss $l$).
We assume $\forall z_n \in D_n$ in any client $C_n$, the loss function $l(z_n;{\bf v}_n)$ is smooth. 
We further assume there exist positive $\beta_{\min}$, $\beta_{\max}$ such that  $\forall z_n \in D_n$, $\forall {\bf v}_n$, $\beta_{\min} {\bf I} \leq \nabla^2_{{\bf v}_n} l(z_n; {\bf v}_n) \leq \beta_{\max} {\bf I}$, 
where ${\bf I}$ is the identity matrix. 
\begin{lemma}
\vspace{-3mm}
[Smoothness of the class-wise $\lambda$-tilted loss $ \tilde{R}_{n,k}(\lambda; \\ {\bf v}_n)$ 
]
\label{lem:localsmooth}
Under Assumption 1, 
the class-wise tilted loss $ \tilde{R}_{n,k}(\lambda;{\bf v}_n) = \frac{1}{\lambda} \log \big( \frac{1}{|D_{n,k}|} \sum_{z \in {D_{n,k}}} e^{\lambda \cdot l(z; {\bf v}_n)} \big)$ is smooth in the vicinity of the optimal local client model ${\bf v}_n^*(\lambda)$, where ${\bf v}_n^*(\lambda) \in \arg\min_{{\bf v}_n}  \tilde{R}_{n,k}(\lambda; {\bf v}_n)$.
\end{lemma}

\begin{lemma}
\vspace{-3mm}
[Strong convexity of the class-wise $\lambda$-tilted loss $ \tilde{R}_{n,k}(\lambda;{\bf v}_n)$ with positive $\lambda$] 
\label{lem:localstrongcvx}
Under Assumption 1, 
for any $\lambda>0$, the class-wise class-wise tilted loss  $ \tilde{R}_{n,k}(\lambda;{\bf v}_n)$  is a strongly convex function of ${\bf v}_n$. That is, for $\lambda>0$, $\nabla_{{\bf v}_n}^2  \tilde{R}_{n,k}(\lambda;{\bf v}_n) > \beta_{min} {\bf I}$. \footnote{The proofs of the above two lemmas are from \cite{li2021tilted}.} 
\vspace{-1mm}
\end{lemma}

Now, we first show the connection between strong convexity and PL inequality and then show that the two-level $(\tau, \lambda)$-titled client loss $\tilde{R}_{n}(\tau,\lambda;{\bf v}_n)$ and the local objective $L_n({\bf v}_n, {\bf w})$ are also smooth and strongly convex.

\begin{lemma}[Strong convexity implies PL inequality]
\label{lem:cvx_LP}
If function $f$ is $\mu$-strongly convex, it satisfies PL inequality with $\mu$. 
\vspace{-4mm}
\end{lemma}

\begin{lemma}
\vspace{-2mm}
[Smoothness of the 
$(\tau,\lambda)$-tilted 
client loss $\tilde{R}_{n}(\tau,\lambda;{\bf v}_n)$ 
and local objective $L_n({\bf v}_n, {\bf w})$ for a given ${\bf w}$
]
\label{lem:globalsmooth}
Under Assumption 1 and based on Lemma~\ref{lem:localsmooth}, the two-level tilted client loss $\tilde{R}_{n}(\tau,\lambda;{\bf v}_n) = \frac{1}{\tau} \log \big(\frac{1}{|D_n|}  \sum_{D_{n,k} \in [D_n]} |D_{n,k}| e^{\tau \cdot  \tilde{R}_{n,k}(\lambda;{\bf v}_n)} \big)$ 
is smooth in the vicinity of the optimal local client model ${\bf v}_n^* (\tau,\lambda)$, where ${\bf v}_n^*(\tau,\lambda) \in \arg\min_{{\bf v}_n}
\tilde{R}_{n}(\tau,\lambda; {\bf v}_n)
$. 
Moreover, the local objective 
$L_n({\bf v}_n, {\bf w})$ for any given ${\bf w}$ is also smooth.
\vspace{-2mm}
\end{lemma}

\begin{lemma}
[Strong convexity of the client loss $\tilde{R}_{n}(\tau,\lambda;{\bf v}_n)$ and local objective $L_n({\bf v}_n, {\bf w})$ for a given ${\bf w}$ with positive $\tau$ and $\lambda$] 
\vspace{-1mm}
\label{lem:globalstrongcvx}
Under Assumption 1 
and Lemma~\ref{lem:localstrongcvx}, for any $\tau,\lambda>0$, the 
client loss $\tilde{R}_{n}(\tau,\lambda;{\bf v}_n)$ and local objective $L_n({\bf v}_n, {\bf w})$ are strong convex functions of ${\bf v}_n$.
For $\tau,\lambda>0$, $\nabla_{{\bf v}_n}^2 \tilde{R}_{n}(\lambda,\tau;{\bf v}_n) > \beta_{min} {\bf I},\nabla_{{\bf v}_n}^2 {L}_{n}({\bf v}_n, {\bf w}) > (\beta_{min} + \mu){\bf I}$. 
\vspace{-2mm}
\end{lemma}

Next, we will first introduce the following theorem and then have the lemma that 
shows the convergence result when either client model ${\bf v}_n$ or global model ${\bf w}$ is fixed.

\begin{theorem}[Karimi et al.\cite{karimi2016linear}]
\label{thm:karimi}
\vspace{-2mm}
For an unconstrained optimization problem $\arg\min_{x} f(x)$, where $f$ is $L$-smooth 
and satisfies the PL inequality with constant $\mu$. 
Then the gradient descent method with a step-size of $1/L$, i.e., $x^{t+1} = x^t - \frac{1}{L} \nabla f(x^t)$, has a global linear convergence rate, i.e., 
$f(x^t) - f(x^*) \leq (1-\frac{\mu}{L})^t (f(x^0) - f(x^*))$.
\vspace{-2mm}
\end{theorem}

\begin{lemma}
\label{lem:cvg_onevariable}
Under Assumption 1
and based on Lemmas~\ref{lem:cvx_LP}-\ref{lem:globalstrongcvx} and Theorem~\ref{thm:karimi}, 
we have: 1) For any given ${\bf w}$, 
$\exists B_1, B_2, B_3 < +\infty$ that do not depend on $\tau$ and $\lambda$ such that 
$\forall \tau,\lambda>0$, after $t$ iterations of gradient descent with the step size $\alpha=\frac{1}{B_1+\tau B_2+\lambda B_3}$, 
$L_{n}({\bf v}_n^t, {\bf w}) - L_{n}({\bf v}_n^*, {\bf w}) \leq \big(1 - \frac{\beta_{\min} + \mu}{B_1 + \tau B_2 + \lambda B_3} \big)^{t} (L_{n}({\bf v}_n^0, {\bf w})- L_{n}({\bf v}_n^*, {\bf w}))$, where ${\bf v}_n^t$ means the updated client model ${\bf v}_n$ in the $t$-th iteration.
2) For any given ${\bf v}_n$, 
$\exists C_1, C_2, C_3 < +\infty$ that do not depend on $\tau$ and $\lambda$ such that for any $\tau,\lambda>0$, after $t$ iterations of gradient descent with the step size $\beta=\frac{1}{C_1+\tau C_2+\lambda C_3}$, 
$L_{n}({\bf v}_n, {\bf w}^t) - L_{n}({\bf v}_n, {\bf w}^*) \leq \big(1 - \frac{\mu}{C_1 + \tau C_2 + \lambda C_3} \big)^{t} (L_{n}({\bf v}_n, {\bf w}^0)- L_{n}({\bf v}_n, {\bf w}^*))$, where ${\bf w}^t$ means the updated global model ${\bf w}$ in the $t$-th iteration.
\vspace{-2mm}
\end{lemma}

Finally, we show the convergence result of FedTilt. We first state two assumptions also used in Ditto~\cite{li2021ditto}. 

\noindent {\bf Assumption 2.}
The global model converges at rate $g(t)$. $\exists g(t)$ s.t. $\lim_{t\rightarrow \infty } g(t)=0, \|{\bf w}^t - {\bf w}^*\|^2 \leq g(t)$. 
E.g., the global model for FedAvg converges with rate $O(1/t)$~\cite{li2019convergence}.

\noindent {\bf Assumption 3.}
Distance between the optimal (initial) client models (i.e., ${\bf v}_n^*, {\bf v}_n^0$) and the optimal (initial) global model (i.e., ${\bf w}^*, {\bf w}^0$) are bounded and ${\bf w}^t, \forall t$ is also norm-bounded.
\begin{theorem}[Convergence result on the client models]
\label{thm:clientcvg}
\vspace{-2mm}
Under Lemma~\ref{lem:cvg_onevariable} and Assumptions 2\&3, 
any $\tau,\lambda>0$, after $t$ iterations of gradient descent with step size $\alpha$ and $\beta$, 
$L_{n}({\bf v}_n^t, {\bf w}^t) - L_{n}({\bf v}_n^*, {\bf w}^*) \leq (D+\frac{\mu}{2} g(t)) \Lambda^{t}  +  E \Gamma^t$, where $\Lambda=(1-\frac{\beta_{\min}+\mu}{B_1+\tau B_2 + \lambda C_3})$, $\Gamma = (1-\frac{\mu}{C_1+\tau C_2 + \lambda C_3})$ and 
$D$ and $E$ are constants defined hereafter. 
\vspace{-2mm}
\end{theorem}

{\bf Theorem}~\ref{thm:clientcvg} indicates that solving the tilted ERM local objective to a local optimum using the gradient-based method in Algorithm~\ref{app:Algorithm} is as efficient as traditional ERM objective.

\vspace{-2mm}
\section{Results}
\subsection{Evaluations on A Toy Example}
\label{sec:evaltoy}

This section explores the fairness and robustness of FedTilt
on a toy example, where we consider federated logistic regression for binary classification. 
 For simplicity,  we consider two clients and client data are sampled from Gaussian distributions. 
This example serves as motivating examples to the theoretical analysis of the framework. By default, we set 
$q=0$ and use \texttt{dist} as the Euclidean distance. Details of the setup and results (Figure ~\ref{fig:toyLogReg}) are in the Appendix. 

Our first experiment focuses on \emph{client fairness} with $\tau=1$ and $\lambda=1$. 
The two clients have very close (and high) test accuracy with different distributions---indicating the client fairness is achieved. In each client, we sample 100 data points from the both classes to form the training set and 20 data points each for testing (Figure ~\ref{fig:toyLogReg}).

Our second experiment focuses on \emph{both client fairness and client data fairness}. We sample {\bf 150} data points from the first distribution, but {\bf only 50} data points from the second distribution for training, and sample 30 and 10 data points respectively from the two distributions for testing. Two clients still achieve very close (and high) test accuracy, as well as high test accuracy per class when $\tau=100$, i.e., the boundaries can well separate the two classes, indicating client fairness and client data fairness are achieved with relatively larger positive $\tau$, which is consistent with Table~\ref{tbl:tiltsummary} (Figure ~\ref{fig:toyLogReg}).

Our third experiment shows FedTilt's performance on {\bf both client, client data fairness, and robustness.} Class 1 in each client has a high variance to induce outliers. We further generate outliers by adding random Gaussian noises (mean 0 and deviation 0.15) to $10\%$ of the samples from class 1. 
The same number of data points as in the second experiment was used. Results show FedTilt is robust to outliers and achieves both client fairness and client data fairness with a negative $\lambda$, e.g., $\lambda=-100$. That is, the two clients have close testing performance, well separate two classes' data, and the decision boundaries are not affected by the outliers---This is because a negative $\lambda$ can suppress the influence of outliers, as shown in Table~\ref{tbl:tiltsummary}. In contrast, the importance of outliers is magnified with a positive $\lambda$ (Figure ~\ref{fig:toyLogReg}).

\vspace{-2mm}
\subsection{Evaluations on Real Datasets}
\vspace{-1mm}
\label{sec:evalreal}
We evaluate FedTilt on three image datasets: MNIST, FashionMNIST (F-Mnist), and CIFAR10. More details of the experiment setup are included in the appendix.
We use three metrics: \textit{test accuracy}, \textit{client fairness} and \textit{client data fairness}. 

FedTilt is tested in two scenarios: one with clean data (Section~\ref{clean_data});
and the other scenario incorporates a certain fraction of outliers among the data (Section \ref{noisy_data}).

\subsubsection{Results on clean data} 
\label{clean_data}
Three metric results on the three clean datasets vs the tilts $\lambda$ and $\tau$. We have the following observations: 1) 
On MNIST and F-MNIST, a larger positive $\lambda$ and $\tau$ yields the highest test accuracy, the lowest standard deviation for client fairness and the lowest  ($\mu_\sigma$, $\sigma_\sigma$) value for client data fairness. 
Notice that client fairness and client data fairness often mutually enhance. For instance, on MNIST, setting $\tau$ to higher values also improves the contributions of $\lambda$. 
2) CIFAR10 is a more challenge dataset than MNIST and F-MNIST, meaning larger training losses, and we choose a smaller range of $\lambda$ and $\tau$ (i.e., $\lambda, \tau \in [-1,2]$). 
The difference is that, the best performance is now obtained when $\lambda=-0.1$. 
A possible reason may be CIFAR10 contains
``outlier" images---i.e., the images far from the true image distribution. We also test FedTilt with different number of clients selected per round and have similar conclusions
(Figure~\ref{mnist_20c} in appendix).

\begin{figure}[!t]
\vspace{-4mm}
    \subfloat[{Test Acc + Client fairness. }
    %\label{cifar_pixels_cf_tau}}
    ]{\includegraphics[bb=0 0 1200 600, scale=0.098]{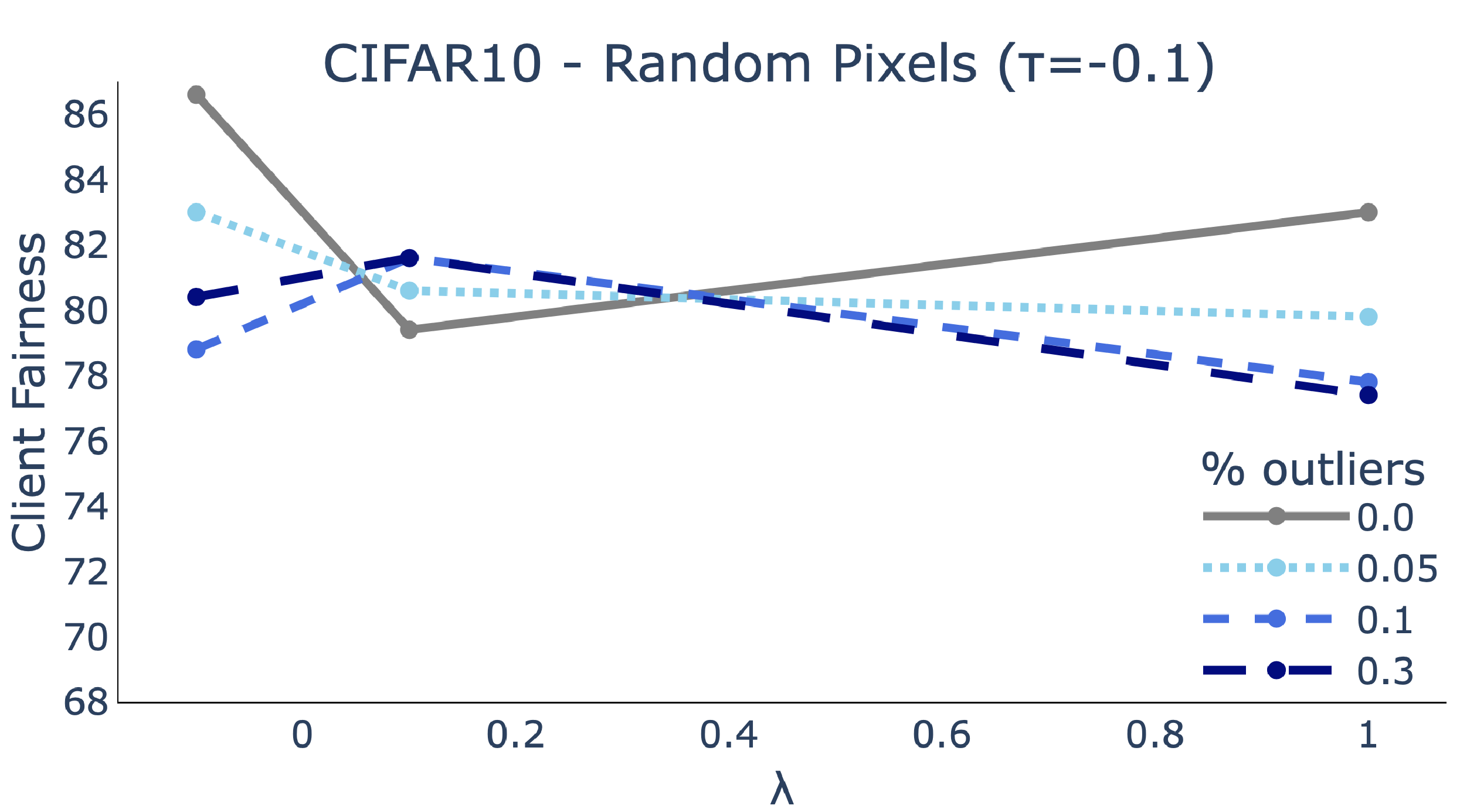}}
    \subfloat[{Client data fairness.  }
    %\label{cifar_pixels_cdf_tau}}
    ]{\includegraphics[bb=0 0 1200 600, scale=0.098]{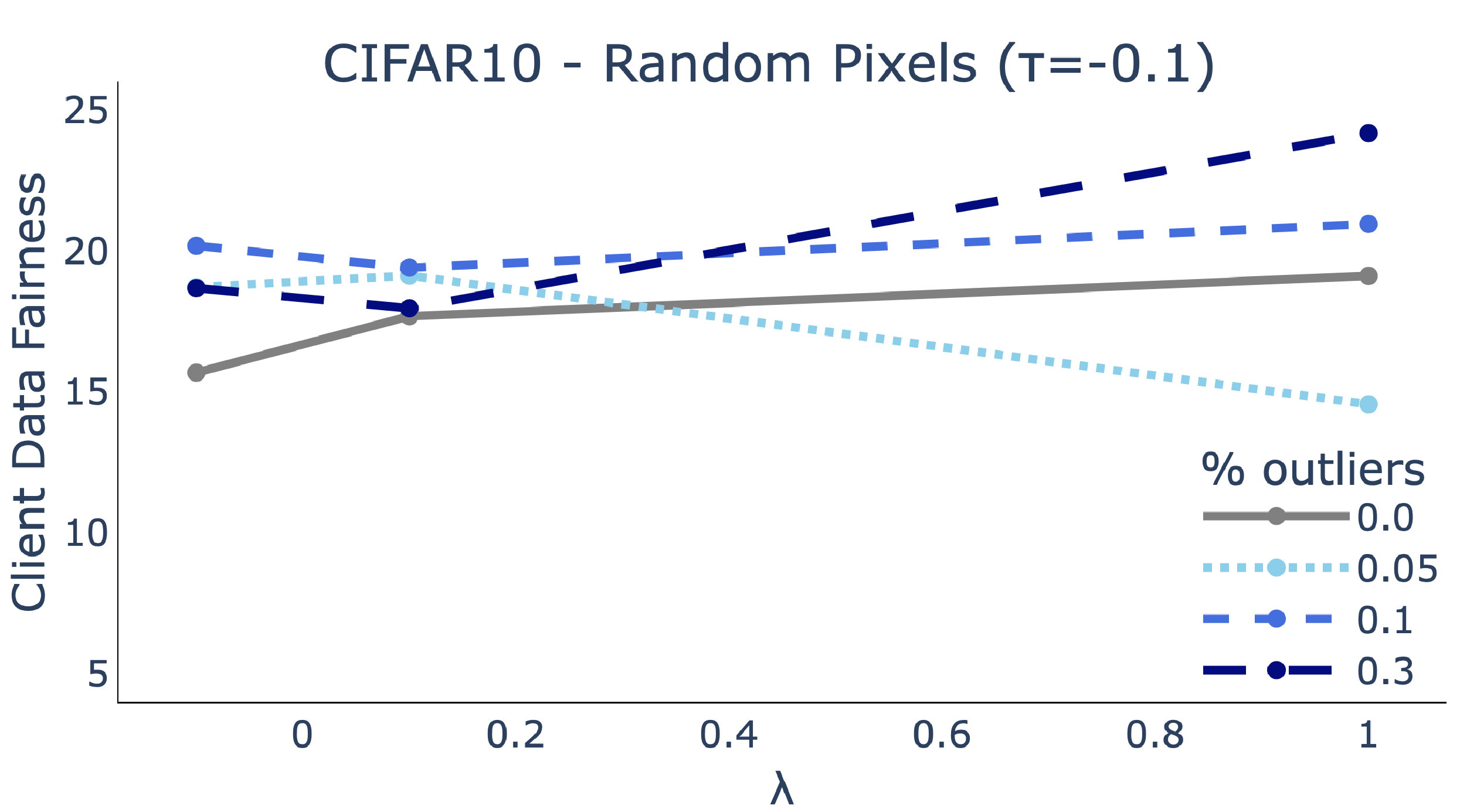}}
    \caption{ CIFAR10 results (c \& d)---persistent random corruptions. Better results obtained with $\lambda=-0.1|0.1$.  
    } \label{mnist_pixels}
    \vspace{-6mm}
\end{figure}

\subsubsection{Results on data with persistent outliers} 
\label{noisy_data}
The second scenario investigates FedTilt's ability to find robust solutions that reduce the effect of persistent outliers---\emph{we inject outliers per global round instead of only once}, to mimic real scenarios, as client data are collected dynamically, and outliers can appear at any time in training. 
We consider \emph{random corruptions}, where 
    30\% 
    pixels of 30\% 
    training samples are corrupted. 
    
Figure~\ref{mnist_pixels} 
 shows the results with persistent random corruptions 
with a fixed $\tau$ vs. $\lambda$. We see FedTilt is robust to persistent random corruptions---its performance is not affected. These results again demonstrate the flexibility and effectiveness of FedTilt in dealing with outliers. 
Figure~\ref{mnist_gaussian}-Figure~\ref{cifar_Gaussian} also show results where data are injected with persistent noises from the standard Gaussian distributions with similar conclusion as the results on persistent random corruptions.

\subsubsection{Comparing with prior works}

This section compares FedTilt with FedAvg and \texttt{Ditto}~\cite{li2021ditto}\footnote{We use the source code of \texttt{Ditto} (\url{https://github.com/litian96/ditto}) and tune the hyperparameters 
to obtain the best possible performance.} on both clean data and data with outliers. 
Since \texttt{Ditto} outperforms other fair FL methods such as TERM~\cite{li2020fair} and FedProx~\cite{li2020federated}, we only consider comparing with \texttt{Ditto} for conciseness. 
All the methods 
are tested with the same settings per dataset. 

\noindent {\bf Results on clean data:}
We found $\lambda=100,\tau=50$ deliver the best performances on clean MNIST and F-MNIST, while $\lambda=1,\tau=2$ the best choice for clean CIFAR10. 
Table \ref{competitiveness_clean} shows the results: 1) FedTilt achieves the best tradeoff among the test accuracy, client fairness, and client data fairness. 
This verifies the benefit of the two-level tilted loss that allows to tune the tilt hyperparameters so that the FedTilt framework can accommodate to very different sets of data. 
2) Though simple, FedAvg can obtain a promising client fairness, even better than \texttt{Ditto}. This indicates that the average aggregation itself can promote client fairness. 

\begin{table}[!t]
\vspace{-2mm}
\centering
\footnotesize
\caption{Clean data} 
\addtolength{\tabcolsep}{-2pt}
\begin{tabular}{llll}
    \hline
    \textbf{MNIST} 
     & Test Acc. & Client fairness & Client data fairness\\
    \hline
    FedAvg & $95.69\%$ & $\sigma=2.91$ & $\mu_\sigma=6.84, \sigma_\sigma=4.90$ \\
    \texttt{Ditto}  & ${\bf 99.25\%}$ & $\sigma=1.27$ & $\mu_\sigma=4.37, \sigma_\sigma=4.23$ \\ 
    FedTilt & $98.53\%$ & $\sigma=1.67$ & ${\bf \mu_\sigma=4.33, \sigma_\sigma=3.33}$ \\
    [0.1cm]\hline
    \textbf{F-MNIST} & Test Acc. & Client fairness & Client data fairness\\
    \hline
    FedAvg & $93.67\%$ & $\sigma=1.97$ & $\mu_\sigma=11.96, \sigma_\sigma=3.52$ \\
    \texttt{Ditto} & $93.77\%$ & $\sigma=5.30$ & $\mu_\sigma=10.89, \sigma_\sigma=7.18$ \\
    FedTilt & ${\bf 96.35}\%$ & ${\bf \sigma=1.85}$ & ${\bf \mu_\sigma=7.61, \sigma_\sigma=3.06}$ \\
    [0.1cm]\hline
    \textbf{CIFAR10} & Test Acc. & Client fairness & Client data fairness\\
    \hline
    FedAvg & $82.20\%$ & $\sigma=4.58$ & $\mu_\sigma=17.96, \sigma_\sigma=3.88$ \\
    \texttt{Ditto} & $74.15\%$ & $\sigma=9.35$ & $\mu_\sigma=18.62, \sigma_\sigma=3.9$ \\ 
    FedTilt & ${\bf 85.24\%}$ & ${\bf \sigma=3.87}$ & ${\bf \mu_\sigma=15.68, \sigma_\sigma=3.69}$ \\
     \hline
\vspace{-8mm}
\end{tabular}  \label{competitiveness_clean}
\end{table}

\begin{table}[!t]
\vspace{-1mm}
\centering
\footnotesize
\caption{persistent random corruptions}
\addtolength{\tabcolsep}{-2pt}
\begin{tabular}{llll}
    \hline
    \textbf{MNIST} & Test Acc. & Client fairness & Client data fairness\\
    \hline
    FedAvg & $95.60\%$ & $\sigma=2.86$ & $\mu_\sigma=8.31, {\bf \sigma_\sigma=1.99}$ \\
    \texttt{Ditto}  & ${\bf 98.95\%}$ & $\sigma=1.72$ & $\mu_\sigma=3.86, \sigma_\sigma=5.35$ \\ 
    FedTilt & $98.46\%$ & ${\bf \sigma=1.50}$ & ${\bf \mu_\sigma=2.79}, \sigma_\sigma=3.36$ \\
    [0.1cm]\hline
    \textbf{F-MNIST} & Test Acc. & Client fairness & Client data fairness\\
    \hline
    FedAvg & $95.81\%$ & $\sigma=3.96$ & $\mu_\sigma=10.01, \sigma_\sigma=5.35$ \\
    \texttt{Ditto}  & $34.83\%$ & $\sigma=24.37$ & $\mu_\sigma=21.71, \sigma_\sigma=19.93$ \\ 
    FedTilt & ${\bf 95.96\%}$ & ${\bf \sigma=3.16}$ & ${\bf \mu_\sigma=8.96, \sigma_\sigma=4.55}$ \\
    [0.1cm]\hline
    \textbf{CIFAR10} & Test Acc. & Client fairness & Client data fairness\\
    \hline
    FedAvg & $81.70\%$ & $\sigma=2.27$ & $\mu_\sigma=17.81, \sigma_\sigma=2.94$ \\
    \texttt{Ditto}  & $52.73\%$ & $\sigma=4.71$ & $\mu_\sigma=19.02, \sigma_\sigma=3.20$ \\ 
    FedTilt & ${\bf 82.01\%}$ & ${\bf \sigma=2.17}$ & ${\bf \mu_\sigma=17.36, \sigma_\sigma=2.39}$ \\ 
    \hline
    \vspace{-9mm}
\end{tabular} 
 \label{competitiveness_pixels}
\end{table}

\noindent {\bf Results on data with corruptions:} 
Results with pixel corruptions are shown in Table~\ref{competitiveness_pixels}, where we set 30\% random pixels are corrupted. 
Here, $\lambda=1,10,-0.1,\tau=-0.5,-1,-0.1$ are the hyperparameter selection in FedTilt for MNIST, F-MNIST and CIFAR10, respectively. 
Still, FedTilt is the most robust to random pixel corruptions and achieves the best client and client data fairness as well. 
\texttt{Ditto}, is even worse in dealing with this type of outlier---Its test accuracy is very low in both F-MNIST and CIFAR10.  
In contrast, both FedAvg and FedTilt are very stable.  
Table~\ref{competitiveness_gaussian} also shows robustness against data with large Gaussian noises and has similar conclusions. 

\noindent {\bf Comparing FedTilt with prior works on Gaussian noises:}
In FedTilt, $\tau=50, \lambda=-10$ yield the best results for MNIST and F-MNIST, while $\tau=-0.1,\lambda=-0.1$ remain as the best for CIFAR10. 
Table \ref{competitiveness_gaussian} shows: 
1) FedTilt performs the best---most robust to persistent Gaussian noises (i.e., test accuracy is the largest), most fair client performance, and most fair client data performance in the three datasets.  
2) All the compared methods do exhibit robustness to Gaussian noise on MNIST and F-MNIST, but \texttt{Ditto} has a large test accuracy drop on CIFAR10. This indicates the persistent Gaussian noise added to the CIFAR10  data 
can be very harmful for \texttt{Ditto}. The injected noisy data might prevent \texttt{Ditto} from convergence. \texttt{Ditto}'s loss was unstable even with 10,000 global rounds where FedTilt converged within 1,000 rounds.
\subsubsection{Summary of the results} 

\vspace{-1mm}
We summarize the above results and draw conclusions as below. These conclusions can help guide the settings of tilt values in real-world applications.   

\begin{itemize}[leftmargin=*]
\item For simple/sanitized datasets, positive $\lambda$ and $\tau$ can  yield promising test accuracy, client and client data fairness. 

\item For complex/noisy datasets, the best performance is often obtained with a negative $\lambda$ or/and negative $\tau$---In order to suppress the effect caused by outliers. 

\item Two-level fairness and robustness show a tradeoff. By tuning the tilt values of $\lambda$ and $\tau$ under the guidance in Table~\ref{tbl:tiltsummary}, we can often obtain a promising tradeoff. 
\vspace{-2mm}
\end{itemize}

\section{Related Work} 
\label{sec:relatedwork}

\noindent {\bf Fair FL.} 
Fairness is an active topic that has received much attention in the
machine learning community~\cite{dwork2012fairness,cotter2019optimization}. 
Fairness in machine learning is typically defined as the protection of some specific attribute(s)/group(s). 
Recently, fairness has been considered in the FL setting movivated by the heterogeneity of the data across different clients which may cause the
testing performance to vary significantly among these
clients. 
To achieve fairness, recent works aim to ensure that the FL training to not overfit a model to any single client at the expense of others~\cite{mohri2019agnostic,deng2020distributionally,li2020fair,li2020federated,li2021ditto}. 
Mohri et al. \cite{mohri2019agnostic} proposed a minimax optimization scheme, termed Agnostic
Federated Learning (AFL), optimizes for the performance of the single worst client. However, due to computation issues, this method can be only applied at a very small number (usually 2-3) of clients.  
Li et al.\cite{li2020fair,li2020federated} designed two sample reweighting approaches (i.e., $q$-FFL and FedProx) to encourage a more fair performance across clients. 
Particularly, these two methods target upweighting the importance of rare clients. 
However, as shown in~\cite{li2021ditto}, they are not robust as they can easily overfit to
clients with outliers such as large noisy data and corrupted data. 
A few methods~\cite{hu2020fedmgda+,li2021ditto} have been proposed to address this issue. 
Hu et al. \cite{hu2020fedmgda+} proposed FedMGDA+, which integrates minimax optimization and gradient normalization techniques to achieve conventional fairness 
and robustness. 

\noindent {\bf Robust FL.} 
In real-world FL applications, a client could produce a negative impact on the model performance with bad quality data. 
For instance, 
a client could train the local data that contains outliers such as noisy data, mislabeled data, and corrupted data---leading to 
bad/ineffective client models. 
A practical FL system should be robust to
outliers. 
In terms of defenses against outliers, 
A series of methods such as learning in the presence of noisy/corrupted data~\cite{jiang2018mentornet,khetan2018learning,shen2019learning,guo2019learning} 
and robust aggregation ~\cite{blanchard2017machine,chen2017distributed,guerraoui2018hidden,yin2018byzantine,guerraoui2018hidden,chen2018draco,pillutla2019robust,li2019rsa,xie2019zeno,wu2020federated,karimireddy2021learning,farhadkhani2022byzantine,zhang2022fldetector,cao2023fedrecover}  
have been proposed. 
For instance, \cite{blanchard2017machine} proposed Krum, which first identifies a local model update as benign if it is similar to other local model updates, 
where the similarity is measured by Euclidean distance.
Then the server only aggregates the benign model updates. 
 While these strategies can improve robustness, they also produce unfair models
especially in heterogeneous settings~\cite{wang2020attack}. 
\vspace{-2mm}
\section{Conclusion}
This paper proposes FedTilt, a novel fairness-preserving and robust federated learning method. 
FedTilt designs a TERM-inspired global objective and a two-level TERM-inspired local objective per client. Minimizing the two objectives with theoretically-guided tilt values can produce the
client fairness, client data fairness, as well as robustness to
persistent outliers. FedTilt also enjoys the convergence property. 
The empirical results demonstrate FedTilt 
outperforms the state-of-the-art fair or/and robust FL methods. 
Future work includes extending our proposed method to federated learning on graph data~\cite{wang2022graphfl}, investigating its robustness against stronger attacks \cite{yang2024breaking}, and exploring model ownership protection strategies \cite{yang2024fedgmark}.

\bibliographystyle{ieeetr}
\bibliography{main}
\newpage 
\appendix  
%\section{appendix}
\subsection{Algorithm 1}
%\label{app:Algorithm}

\begin{algorithm}[!t]
%\footnotesize
\caption{FedTilt}
\begin{algorithmic}[1]
 \REQUIRE $N$: \#total clients; $\rho$: participating clients\% ; $B$: mini-batch size; $T$: \#global rounds; $E$: \#epochs for intermediate or client model update; $E_2$: \#epochs for global model update; 
 ${D}_n$: client $n$'s training data, $\eta_1, \eta_2, \eta_3$: learning rates
    
    \ENSURE global model ${\bf w}$; personalized client models $\{{\bf v}_n\}$
    \STATE initialize ${\bf w}={\bf w}^0$ and $\{{\bf v}_n^0\}_{[n\in N]}$
    \FOR{each global round $t$ from $1$ to $T$}
        \STATE $m\gets\max(\rho \cdot N,1)$; 
        $M_t \gets$ (random set of $m$ clients)
        \FOR{each client $n \in M_t$ in parallel}
            \STATE ${\bf w}^{t}_n, {\bf v}^{t}_n  \gets$ \textbf{ClientUpdate}$({D}_n, {\bf w}^{t-1}, {\bf v}_n^{t-1})$
        \ENDFOR
        \STATE ${\bf w}^{t}\gets$\textbf{ServerUpdate} $({\bf w}^{t-1},  \{{\bf w}^t_n\}_{n \in M_t})$
    \ENDFOR
    \STATE  return ${\bf w}^{T}$ and $\{{\bf v}_n^T\}_{n \in N}$ \\

    \STATE \textbf{ClientUpdate$({D}_n, {\bf w}^{t-1}, {\bf v}_n^{t-1})$} 
    \FOR {each local epoch $e$ from $1$ to $E$}
       
         \STATE $\mathcal{B} \gets$ (split ${D}_n$ into mini-batches of size $B$)
        \FOR {each batch $b\in \mathcal{B}$}
         \STATE Update intermediate client model ${\bf w}_n^{t}$ given ${\bf w^{t-1}}$: 
         ${\bf w}^{t}_n \gets {\bf w}^{t-1}-\eta_1 \nabla_b \tilde{R}_n(\tau,\lambda;{\bf w}^{t-1})$
        
        \STATE Update personalized client model ${\bf v}_n^{t}$ given ${\bf w^{t-1}}$ and ${\bf v^{t-1}}$: 
        ${\bf v}_n^t \gets {\bf v}_n^{t-1} -\eta_2 \nabla_{b} L_n({\bf v}_n^{t-1}, {\bf w}^{t-1})$
        
        \ENDFOR
    \ENDFOR
    \STATE return ${\bf w}_n^t$ and ${\bf v}_n^t$ 

    \STATE \textbf{ServerUpdate$({\bf w}^{t-1};\{{\bf w}_n^t\}_{n\in M_t})$} // Global model update 
    \FOR{each local epoch $e$ from $1$ to $E_2$}
    
    \STATE ${\bf w}^t \gets {\bf w}^{t-1} -\eta_3 \nabla_{{\bf w}}
       \tilde{R}_G(q;\{{\bf w}^t_n\}_{n \in M_t},{\bf w}^{t-1})$ 
    \ENDFOR
    \STATE 
    return ${\bf w}^t$ 
\end{algorithmic}
\label{app:Algorithm}
\end{algorithm}

\subsection{Background on FedProx and Ditto}
%\label{app:background}
\noindent \textbf{FedProx~\cite{li2020federatedoptimization}.}
In practice, the data distribution across clients can be different. To account for such data heterogeneity that often leads to unfair performance across clients, FedProx proposes to add a proximal term to the local objective. 
Specifically, each client $C_n$ minimizes the local objective as below to learn the shared global model ${\bf w}$:
\begin{align}
& \textbf{Global obj.: } {\bf w} = \arg \min_{\bf w} G({\bf w}, \{{\bf w}_n\}),
\\ 
   &  \textbf{Local obj.: }  {\bf w}_n = \arg\min_{{\bf w}_n}
    L_n({\bf w}_n, {\bf w}) \nonumber \\
    & \qquad \qquad \qquad \, \, = F_n({\bf w}) + 
    \frac{\mu}{2} \|{\bf w}_n - {\bf w} \|^2, 
\end{align}
where the hyperparameter $\mu$ tradeoffs the local objective and the proximal term 
$\|{\bf w}_n - {\bf w} \|^2$, 
which aims to restrict
the intermediate local models ${\bf w}_n$ in each client to be closer to the global model ${\bf w}$, thus mitigating unfairness. The proximal term also shows to improve the stability of training.  
Note that when $\mu=0$, FedProx reduces to the FedAvg ~\cite{mcmahan2017communication}.

\noindent \textbf{Ditto~\cite{li2021ditto}.}
The state-of-the-art Ditto differs other FL methods (e.g., FedAvg and FedProx~\cite{li2020federatedoptimization}) 
by learning personalized client models via federated multi-task learning. Specifically, Ditto considers optimizing both the global objective and local objective and simultaneously learns the global model and a local model (i.e., ${\bf v}_n$) per client $C_n$ as below:
\begin{align}
     & \textbf{Global obj.: } {\bf w}^* \in \arg\min\nolimits_{\bf w} G({\bf w}, \{{\bf w}_n\}),
     \\ 
    & \textbf{Local obj.: }  
    {\bf v}_n^* = \arg
    \min_{{\bf v}_n}L_n({\bf v}_n, {\bf w}^*) \\ 
    & \qquad \qquad \qquad \, \,= F_n({\bf v}_n) + 
    \frac{\mu}{2} \|{\bf v}_n - {\bf w}^* \|^2; \nonumber 
\end{align}
where it uses the average aggregation in $G(\cdot)$ by default and the hyperparameter $\mu$ tradeoffs the local client loss and the closeness between personalized client models and global models (which ensures client fairness). For instance, when $\mu=0$, Ditto reduces to training local client models $\{{\bf v}_n\}$; On the contrary, when $\mu=+\infty$, all client models degenerate to the global model ${\bf w}$, making Ditto recover the FedAvg. Hence, through properly setting $\mu$, Ditto can achieve a promising fairness across clients, and maintain the FL performance as well. 

\begin{definition}[Smooth function]
\vspace{-1mm}
A function $f$ is $L$-smooth, if for all $x$ and $y$, 
$f(y) \leq f(x) + \langle \nabla_x f(x), y-x \rangle + \frac{L}{2} \|y-x\|^2$. 
\end{definition}

\begin{definition}[Strongly convex function]
\vspace{-2mm}
A function $f$ is $\mu$-strongly convex, if for all $x$ and $y$, 
$f(y) \geq f(x) + \langle \nabla_x f(x), y-x \rangle + \frac{\mu}{2} \|y-x\|^2$. In other words, $\nabla^2_x f(x) \geq \mu$. 
\end{definition}

\begin{definition}[Polyak-Lojasiewicz (PL) inequality]
\vspace{-2mm}
\label{def:PL}
A function $f$ satisfies the PL inequality if the following holds for all $x$:
$\frac{1}{2}\|\nabla_x f(x)\|^2 \geq \mu (f(x)-f(x^*))$ for some $\mu>0$,  where $x^* = \arg\min_x f(x)$.
\vspace{-2mm}
\end{definition}

\subsection{Proofs of Propositions}

\begin{proof}
Recall that $$\tilde{R}_G(q; \{{\bf w}_n\}, {\bf w}) = \frac{1}{q} \log \big( \frac{1}{N} \sum_{n \in [N]} e^{q \cdot \texttt{dist}({\bf w}_n, {\bf w})} \big)$$. 
First, by setting $q=0$ and $\texttt{dist}({\bf w}_n, {\bf w})=\|{\bf w}_n-{\bf w}\|^2$,
\begin{small}
\begin{align*}
    & \tilde{R}_G(0; \{{\bf w}_n\},{\bf w}) 
    = \frac{1}{N} \sum_{n \in [N]} \|{\bf w}_n -{\bf w}\|^2 \\
    & = {1}/{N} [ \sum_{n \in [N]} \langle {\bf w}_n, {\bf w}_n \rangle  + N \langle {\bf w}, {\bf w} \rangle  - 2\langle {\bf w}, \sum_{n \in [N]} {\bf w}_n \rangle ] 
\end{align*}
\end{small}

By setting its gradient w.r.t ${\bf w}$ to be 0, we have 
\begin{small}
\begin{align}
\label{eqn:fedtilt_avg}
  & \nabla_{\bf w} \tilde{R}_G(0; \{{\bf w}_n\},{\bf w}) = {1}/{N} [ 2N {\bf w} - 2\sum_{n \in [N]} {\bf w}_n ] = 0 \nonumber \\
  & \Longrightarrow{\bf w} = \frac{1}{N} \sum_{n \in [N]} {\bf w}_n, 
\end{align}
\end{small}
which is exactly the average aggregation. 

Further, by setting $\mu=+\infty$, minimizing the client loss $L_n$ requires  ${\bf v}_n = {\bf w}$. 
Then, with $\tau=0$ and $\lambda=0$ we have the per client loss as 
\begin{small}
\begin{align*}
& L_n({\bf v}_n, {\bf w}) = \tilde{R}_n (0, 0; {\bf w}) = \frac{1}{|D_n|} \sum_{z \in D_n} l(z; {\bf w}) = F_n({\bf w}).
\end{align*} 
\end{small}
Combing it with Equation~\ref{eqn:fedtilt_avg} reaches FedAvg.
\label{app:proofprop}
\end{proof} 

\begin{proof}
Similar to Proprosition 1, with $q=0$ and $\texttt{dist}({\bf w}_n, {\bf w})=\|{\bf w}_n-{\bf w}\|^2$ and by setting the gradient $\nabla_{\bf w} \tilde{R}_G(0; \{{\bf w}_n\},{\bf w})$ to be zero reaches to the  average aggregation. 
Also, with $\tau=0$, $\lambda=0$, and ${\bf v}_n = {\bf w}_n$, the local loss $L_n({\bf w}_n, {\bf w}) = \tilde{R}_n (0, 0; {\bf w}) + \mu/2 \cdot \|{\bf w}_n - {\bf w} \|^2 \rightarrow \frac{1}{|D_n|} \sum_{z \in D_n} l(z; {\bf w}_n) + \mu/2 \cdot \|{\bf w} - {\bf w} \|^2$, which is  the objective of FedProx in Equation~\ref{localobj:fedprox}.
\end{proof} 

\begin{proof}
Similarly, with 
$q=0$ and $\texttt{dist}({\bf w}_n, {\bf w})=\|{\bf w}_n-{\bf w}\|^2$ 
and by setting the gradient $\nabla_{\bf w} \tilde{R}_G(0; \{{\bf v}_n\},{\bf w})$ to be zero reaches to the average aggregation.  
Moreover, with 
$\tau=0$, and $\lambda=0$, the local client loss becomes $L_n({\bf v}_n, {\bf w}) = \tilde{R}_n (0, 0; {\bf v}_n) + \mu/2 \cdot \|{\bf v}_n - {\bf w} \|^2 \rightarrow \frac{1}{|D_n|} \sum_{z \in D_n} l(z; {\bf v}_n) + \mu/2 \cdot \|{\bf v}_n - {\bf w} \|^2$, which is  the objective of Ditto in Equation~\ref{obj:ditto_local}.
\end{proof}

\subsection{Convergence Results of FedTilt}
We first introduce the following definitions, assumptions, and lemmas.
Then we proof the convergence conditions of FedTilt. 

The overall proof idea is as follows: 1) Assume that standard loss $l$ is convex and strongly smooth, a standard assumption used in most FL methods~\cite{li2021tilted,li2021ditto,li2020federated,li2019convergence};
2) Show the class-wise one-level $\lambda$-tilted loss $ \tilde{R}_{n,k}(\lambda;{\bf v}_n)$ is convex and smooth based on 1); 
3) Further show the two-level $(\tau, \lambda)$-tilted client loss $\tilde{R}_{n}(\tau, \lambda; {\bf v}_n)$ and local objective $L_n({\bf v}_n,  {\bf w})$ are convex and smooth based on 1) and 2);
4) Show the global loss is convergent based on Ditto~\cite{li2021ditto}.  
5) Finally, combining the convergence property of local objective and global objective, we show the convergence condition of FedTilt. 

\begin{definition}[Smooth function]
A function $f$ is $L$-smooth, if for all $x$ and $y$, 
$f(y) \leq f(x) + \langle \nabla_x f(x), y-x \rangle + \frac{L}{2} \|y-x\|^2$. 
\end{definition}

\begin{definition}[Strongly convex function]
A function $f$ is $\mu$-strongly convex, if for all $x$ and $y$, 
$f(y) \geq f(x) + \langle \nabla_x f(x), y-x \rangle + \frac{\mu}{2} \|y-x\|^2$. In other words, $\nabla^2_x f(x) \geq \mu$. 
\end{definition}

\begin{definition}[Polyak-Lojasiewicz (PL) inequality]
A function $f$ satisfies the PL inequality if the following holds for all $x$:
$\frac{1}{2}\|\nabla_x f(x)\|^2 \geq \mu (f(x)-f(x^*))$ for some $\mu>0$,  where $x^* = \arg\min_x f(x)$.
\end{definition}

\noindent {\bf Assumption 1} (Smooth and strongly convex loss $l$).
We assume $\forall z_n \in D_n$ in any client $C_n$, the loss function $l(z_n;{\bf v}_n)$ is smooth. 
We further assume there exist positive $\beta_{\min}$, $\beta_{\max}$ such that  $\forall z_n \in D_n$ and any ${\bf v}_n$, $\beta_{\min} {\bf I} \leq \nabla^2_{{\bf v}_n} l(z_n; {\bf v}_n) \leq \beta_{\max} {\bf I}$, 
where ${\bf I}$ is the identity matrix. \\

\subsection{Proofs for Lemma 3-6}
\subsubsection{Proof for Lemma 3}
\begin{proof}
Proof for Lemma For a $\mu$-strongly convex function $f$, we have 
$f(y) \geq f(x) + \langle \nabla_x f(x), y-x \rangle + \frac{\mu}{2} \|y-x\|^2$, $\forall x, y$. 
Now we minimize both LHS and RHS and note that minimization kepng the inequality. 
By minimizing the LHS $f(y)$ we have $\min_y f(y) = f(x^*)$.  
To solve the RHS, we set the gradient of $f$ w.r.t. $y$ to be 0, and have $\nabla_x f(x) + \mu (y-x) = 0$, which impliles $y=x - \frac{1}{\mu} \nabla_x f(x)$. Substituting $y$ in the RHS, which becomes $f(x) - \frac{1}{\mu} \|\nabla_x f(x)\|^2 + \frac{1}{2\mu} \|\nabla_x f(x)\|^2$. 
Then, as $\min LHS \geq \min RHS$, we have $f(x^*) \geq f(x) - \frac{1}{2\mu} \|\nabla_x f(x)\|^2$ and then $\frac{1}{2}\|\nabla_x f(x)\|^2 \geq \mu (f(x)-f(x^*))$. 
\end{proof}

\subsubsection{Proof for Lemma 4}
\begin{proof}
The main idea follows the proof for Lemma~\ref{lem:localsmooth}. Particularly, 
the class-wise tilted loss $ \tilde{R}_{n,k}(\lambda;{\bf v}_n)$ is the tilted version of the conventional loss $l$ and 
Lemma~\ref{lem:localsmooth} requires the loss $l$ to be smooth and strongly convex based on Assumption 1.
Similarly, the  two-level tilted client loss $\tilde{R}_{n}(\tau,\lambda;{\bf v}_n)$ is the tilted version of  the class-wise tilted loss $ \tilde{R}_{n,k}(\lambda;{\bf v}_n)$, and hence we require it 
to be smooth and strongly convex, which are verified in Lemma~\ref{lem:localsmooth} and Lemma~\ref{lem:localstrongcvx}.   
Furthermore, the local objective 
$L_n({\bf v}_n, {\bf w})= \tilde{R}_{n}(\lambda,\tau;{\bf v}_n)  + \mu/2 \|{\bf v}_n - {\bf w}\|^2$ is naturally a smoothed version of $\tilde{R}_{n}(\lambda,\tau;{\bf v}_n)$ for any given ${\bf w}$, thus completing the proof. 
\vspace{-2mm}
\end{proof}

\subsubsection{Proof for Lemma 5}
\begin{proof}
The main idea follows the proof for Lemma~\ref{lem:localstrongcvx}. Particularly, Lemma~\ref{lem:localstrongcvx} requires the loss $l$ to be strongly convex based on Assumption 1, 
where we require the class-wise loss  $ \tilde{R}_{n,k}(\lambda;{\bf v}_n)$ to be strongly convex, which is verified in
Lemma~\ref{lem:localstrongcvx}. 
Note that when $\nabla^2_{{\bf v}_n} l(z_n; {\bf v}_n) \geq \beta_{\min} {\bf I}$, Lemma~\ref{lem:localstrongcvx} has $\nabla_{{\bf v}_n}^2  \tilde{R}_{n,k}(\lambda;{\bf v}_n) > \beta_{min} {\bf I}$ 
$\forall \lambda>0$. 
Based on this,
$\forall \tau>0,\lambda>0$, $\nabla_{{\bf v}_n}^2 \tilde{R}_{n}(\lambda,\tau;{\bf v}_n) > \beta_{min} {\bf I}$.
As $L_n({\bf v}_n, {\bf w})= \tilde{R}_{n}(\lambda,\tau;{\bf v}_n)  + \mu/2 \|{\bf v}_n - {\bf w}\|^2$, we have  $\nabla_{{\bf v}_n}^2 \tilde{L}_{n}({\bf v}_n, {\bf w}) > (\beta_{min} + \mu){\bf I}$ for a fixed ${\bf w}$. 
\end{proof}

\subsubsection{Proof for Lemma 6}
\begin{proof}
First, we observe that the local objective 
$L_{n}({\bf v}_n, {\bf w})$ is $(\beta_{\min} + \mu)$-strongly convex for any given ${\bf w}$ and all $\tau,\lambda>0$ from {\bf Lemma}~\ref{lem:globalstrongcvx}. 
Based on {\bf Lemma}~\ref{lem:cvx_LP}, $L_{n}({\bf v}_n, {\bf w})$ with a given ${\bf w} $ also satisfies the PL inequality with constant $(\beta_{\min} + \mu)$. 
Next, noticed by {\bf Lemma}~\ref{lem:globalsmooth} and the proof for {\bf Lemma}~\ref{lem:localsmooth}, there exist $B_1, B_2, B_3 < +\infty$ such that $L_{n}({\bf v}_n, {\bf w})$ is $(B_1 + \tau B_2 + \lambda B_3)$-smooth 
for all $\tau,\lambda>0$ and a given ${\bf w}$. 
Now, using the gradient descent method to optimize $L_{n}({\bf v}_n, {\bf w})$ with a fixed ${\bf w}$, we have the convergence result of 1) based on {\bf Theorem}~\ref{thm:karimi}.
Similarly, 
the local objective 
$L_{n}({\bf v}_n, {\bf w})$ is $\mu$-strongly convex for any given ${\bf v}_n$ and all $\tau,\lambda>0$ from {\bf Lemma}~\ref{lem:globalstrongcvx}. 
Based on {\bf Lemma}~\ref{lem:cvx_LP}, $L_{n}({\bf v}_n, {\bf w})$ with a given ${\bf v}_n$ also satisfies the PL inequality with constant $\mu$. 
Next, noticed by {\bf Lemma}~\ref{lem:globalsmooth},
there exist $C_1, C_2, C_3 < +\infty$ such that $L_{n}({\bf v}_n, {\bf w})$ is $(C_1 + \tau C_2 + \lambda C_3)$-smooth 
for all $\tau,\lambda>0$ for a given ${\bf v}_n$. 
Now, using the gradient descent method to optimize $L_{n}({\bf v}_n, {\bf w})$ with a fixed ${\bf v}_n$, we have the convergence result of 2) based on {\bf Theorem}~\ref{thm:karimi}. 
\vspace{-2mm}
\end{proof}
\begin{lemma}
[Smoothness of the class-wise $\lambda$-tilted loss $ \tilde{R}_{n,k}(\lambda;{\bf v}_n)$ 
]
Under Assumption 1, 
the class-wise tilted loss $ \tilde{R}_{n,k}(\lambda;{\bf v}_n) = \frac{1}{\lambda} \log \big( \frac{1}{|D_{n,k}|} \\ \sum_{z \in {D_{n,k}}} e^{\lambda \cdot l(z; {\bf v}_n)} \big)$ is smooth in the vicinity of the optimal local client model ${\bf v}_n^*(\lambda)$, where ${\bf v}_n^*(\lambda) \in \arg\min_{{\bf v}_n}  \tilde{R}_{n,k}(\lambda; {\bf v}_n)$.
\end{lemma}

\begin{lemma}
[Strong convexity of the class-wise $\lambda$-tilted loss $ \tilde{R}_{n,k}(\lambda;{\bf v}_n)$ with positive $\lambda$] 
Under Assumption 1, 
for any $\lambda>0$, the class-wise class-wise tilted loss  $ \tilde{R}_{n,k}(\lambda;{\bf v}_n)$  is a strongly convex function of ${\bf v}_n$. That is, for $\lambda>0$, $\nabla_{{\bf v}_n}^2  \tilde{R}_{n,k}(\lambda;{\bf v}_n) > \beta_{min} {\bf I}$. 
\end{lemma}
\noindent The proofs of the above two lemmas are from \cite{li2021tilted}.

Now, we first show the connection between strong convexity and PL inequality and then show that the two-level $(\tau, \lambda)$-titled client loss $\tilde{R}_{n}(\tau,\lambda;{\bf v}_n)$ and the local objective $L_n({\bf v}_n, {\bf w})$ are also smooth and strongly convex.

\begin{lemma}[Strong convexity implies PL inequality]

If a function $f$ is $\mu$-strongly convex, it satisfies the PL inequality with the same $\mu$. 

\end{lemma}
\begin{proof}
For a $\mu$-strongly convex function $f$, we have 
$f(y) \geq f(x) + \langle \nabla_x f(x), y-x \rangle + \frac{\mu}{2} \|y-x\|^2$, $\forall x, y$. 
Now we minimize both LHS and RHS and note that minimization kepng the inequality. 
By minimizing the LHS $f(y)$ we have $\min_y f(y) = f(x^*)$.  
To solve the RHS, we set the gradient of $f$ w.r.t. $y$ to be 0, and have $\nabla_x f(x) + \mu (y-x) = 0$, which impliles $y=x - \frac{1}{\mu} \nabla_x f(x)$. Substituting $y$ in the RHS, which becomes $f(x) - \frac{1}{\mu} \|\nabla_x f(x)\|^2 + \frac{1}{2\mu} \|\nabla_x f(x)\|^2$. 
Then, as $\min LHS \geq \min RHS$, we have $f(x^*) \geq f(x) - \frac{1}{2\mu} \|\nabla_x f(x)\|^2$ and then $\frac{1}{2}\|\nabla_x f(x)\|^2 \geq \mu (f(x)-f(x^*))$. 
\end{proof}

\begin{lemma}
[Smoothness of the 
$(\tau,\lambda)$-tilted 
client loss $\tilde{R}_{n}(\tau,\lambda;{\bf v}_n)$ 
and local objective $L_n({\bf v}_n, {\bf w})$ for a given ${\bf w}$
]

Under Assumption 1
and based on Lemma~\ref{lem:localsmooth}, the two-level tilted client loss $\tilde{R}_{n}(\tau,\lambda;{\bf v}_n) = \frac{1}{\tau} \log \big(\frac{1}{|D_n|}  \sum_{D_{n,k} \in [D_n]} |D_{n,k}| e^{\tau \cdot  \tilde{R}_{n,k}(\lambda;{\bf v}_n)} \big)$ 
is smooth in the vicinity of the optimal local client model ${\bf v}_n^*(\tau,\lambda)$, where ${\bf v}_n^*(\tau,\lambda) \in \arg\min_{{\bf v}_n} \\
\tilde{R}_{n}(\tau,\lambda; {\bf v}_n)
$. 
Moreover, the local objective 
$L_n({\bf v}_n, {\bf w})$ for any given ${\bf w}$ is also smooth.
\end{lemma}
\begin{proof}
The main idea follows the proof for Lemma~\ref{lem:localsmooth}. Particularly, 
the class-wise tilted loss $ \tilde{R}_{n,k}(\lambda;{\bf v}_n)$ is the tilted version of the conventional loss $l$ and 
Lemma~\ref{lem:localsmooth} requires the loss $l$ to be smooth and strongly convex based on Assumption 1. 
Similarly, the  two-level tilted client loss $\tilde{R}_{n}(\tau,\lambda;{\bf v}_n)$ is the tilted version of  the class-wise tilted loss $ \tilde{R}_{n,k}(\lambda;{\bf v}_n)$, and hence we require it 
to be smooth and strongly convex, which are verified in Lemma~\ref{lem:localsmooth} and Lemma~\ref{lem:localstrongcvx}.   
Furthermore, the local objective 
$L_n({\bf v}_n, {\bf w})= \tilde{R}_{n}(\lambda,\tau;{\bf v}_n)  + \mu/2 \|{\bf v}_n - {\bf w}\|^2$ is naturally a smoothed version of $\tilde{R}_{n}(\lambda,\tau;{\bf v}_n)$ for any given ${\bf w}$, thus completing the proof. 
\end{proof}

\begin{lemma}
[Strong convexity of the client loss $\tilde{R}_{n}(\tau,\lambda;{\bf v}_n)$ and local objective $L_n({\bf v}_n, {\bf w})$ for a given ${\bf w}$ with positive $\tau$ and $\lambda$] 
Under Assumption 1 
and Lemma~\ref{lem:localstrongcvx}, for any $\tau,\lambda>0$, the 
client loss  $\tilde{R}_{n}(\tau,\lambda;{\bf v}_n)$ and local objective $L_n({\bf v}_n, {\bf w})$ are a  strongly convex function of ${\bf v}_n$.
More specifically, for $\tau>0,\lambda>0$, $\nabla_{{\bf v}_n}^2 \tilde{R}_{n}(\lambda,\tau;{\bf v}_n) > \beta_{min} {\bf I}$ and $\nabla_{{\bf v}_n}^2 {L}_{n}({\bf v}_n, {\bf w}) > (\beta_{min} + \mu){\bf I}$. 
\end{lemma}

\begin{proof}
The main idea follows the proof for Lemma~\ref{lem:localstrongcvx}. Particularly, Lemma~\ref{lem:localstrongcvx} requires the loss $l$ to be strongly convex based on Assumption 1, 
where we require the class-wise loss  $ \tilde{R}_{n,k}(\lambda;{\bf v}_n)$ to be strongly convex, which is verified in
Lemma~\ref{lem:localstrongcvx}. 
Note that when $\nabla^2_{{\bf v}_n} l(z_n; {\bf v}_n) \geq \beta_{\min} {\bf I}$, Lemma~\ref{lem:localstrongcvx} has $\nabla_{{\bf v}_n}^2  \tilde{R}_{n,k}(\lambda;{\bf v}_n) > \beta_{min} {\bf I}$ 
$\forall \lambda>0$. 
Based on this,
$\forall \tau>0,\lambda>0$, $\nabla_{{\bf v}_n}^2 \tilde{R}_{n}(\lambda,\tau;{\bf v}_n) > \beta_{min} {\bf I}$.
As $L_n({\bf v}_n, {\bf w})= \tilde{R}_{n}(\lambda,\tau;{\bf v}_n)  + \mu/2 \|{\bf v}_n - {\bf w}\|^2$, we have  $\nabla_{{\bf v}_n}^2 \tilde{L}_{n}({\bf v}_n, {\bf w}) > (\beta_{min} + \mu){\bf I}$ for a fixed ${\bf w}$. 
\end{proof}

Next, we will first introduce the following theorem and then have the lemma that 
shows the convergence result when either client model ${\bf v}_n$ or global model ${\bf w}$ is fixed.

\begin{theorem}[Karimi et al.\cite{karimi2016linear}]
For an unconstrained optimization problem $\arg\min_{x} f(x)$, where $f$ is $L$-smooth 
and satisfies the PL inequality with constant $\mu$. 
Then the gradient descent method with a step-size of $1/L$, i.e., $x^{t+1} = x^t - \frac{1}{L} \nabla f(x^t)$, has a global linear convergence rate, i.e., 
$f(x^t) - f(x^*) \leq (1-\frac{\mu}{L})^t (f(x^0) - f(x^*))$.
\end{theorem}

\begin{lemma}
Under Assumption 1
and based on Lemmas~\ref{lem:cvx_LP}-\ref{lem:globalstrongcvx} and Theorem~\ref{thm:karimi}, 
we have: 1) For any given ${\bf w}$, 
$\exists B_1, B_2, B_3 < +\infty$ that do not depend on $\tau$ and $\lambda$ such that 
$\forall \tau,\lambda>0$, after $t$ iterations of gradient descent with the step size $\alpha=\frac{1}{B_1+\tau B_2+\lambda B_3}$, 
$L_{n}({\bf v}_n^t, {\bf w}) - L_{n}({\bf v}_n^*, {\bf w}) \leq \big(1 - \frac{\beta_{\min} + \mu}{B_1 + \tau B_2 + \lambda B_3} \big)^{t} (L_{n}({\bf v}_n^0, {\bf w})- L_{n}({\bf v}_n^*, {\bf w}))$, where ${\bf v}_n^t$ means the updated client model ${\bf v}_n$ in the $t$-th iteration.
2) For any given ${\bf v}_n$, 
$\exists C_1, C_2, C_3 < +\infty$ that do not depend on $\tau$ and $\lambda$ such that for any $\tau,\lambda>0$, after $t$ iterations of gradient descent with the step size $\beta=\frac{1}{C_1+\tau C_2+\lambda C_3}$, 
$L_{n}({\bf v}_n, {\bf w}^t) - L_{n}({\bf v}_n, {\bf w}^*) \leq \big(1 - \frac{\mu}{C_1 + \tau C_2 + \lambda C_3} \big)^{t} (L_{n}({\bf v}_n, {\bf w}^0)- L_{n}({\bf v}_n, {\bf w}^*))$, where ${\bf w}^t$ means the updated global model ${\bf w}$ in the $t$-th iteration.
\end{lemma}
\begin{proof}
First, we observe that the local objective 
$L_{n}({\bf v}_n, {\bf w})$ is $(\beta_{\min} + \mu)$-strongly convex for any given ${\bf w}$ and all $\tau,\lambda>0$ from {\bf Lemma}~\ref{lem:globalstrongcvx}. 
Based on {\bf Lemma}~\ref{lem:cvx_LP}, $L_{n}({\bf v}_n, {\bf w})$ with a given ${\bf w} $ also satisfies the PL inequality with constant $(\beta_{\min} + \mu)$. 
Next, noticed by {\bf Lemma}~\ref{lem:globalsmooth} and the proof for {\bf Lemma}~\ref{lem:localsmooth}, there exist $B_1, B_2, B_3 < +\infty$ such that $L_{n}({\bf v}_n, {\bf w})$ is $(B_1 + \tau B_2 + \lambda B_3)$-smooth 
for all $\tau,\lambda>0$ and a given ${\bf w}$. 
Now, using the gradient descent method to optimize $L_{n}({\bf v}_n, {\bf w})$ with a fixed ${\bf w}$, we have the convergence result of 1) based on {\bf Theorem}~\ref{thm:karimi}.
Similarly, 
the local objective 
$L_{n}({\bf v}_n, {\bf w})$ is $\mu$-strongly convex for any given ${\bf v}_n$ and all $\tau,\lambda>0$ from {\bf Lemma}~\ref{lem:globalstrongcvx}. 
Based on {\bf Lemma}~\ref{lem:cvx_LP}, $L_{n}({\bf v}_n, {\bf w})$ with a given ${\bf v}_n$ also satisfies the PL inequality with constant $\mu$. 
Next, noticed by {\bf Lemma}~\ref{lem:globalsmooth},
there exist $C_1, C_2, C_3 < +\infty$ such that $L_{n}({\bf v}_n, {\bf w})$ is $(C_1 + \tau C_2 + \lambda C_3)$-smooth 
for all $\tau,\lambda>0$ for a given ${\bf v}_n$. 
Now, using the gradient descent method to optimize $L_{n}({\bf v}_n, {\bf w})$ with a fixed ${\bf v}_n$, we have the convergence result of 2) based on {\bf Theorem}~\ref{thm:karimi}. 
\end{proof}

\noindent {\bf Proof of Theorem 3}

\begin{proof}
$L_{n}({\bf v}_n^t, {\bf w}^t) - L_{n}({\bf v}_n^*, {\bf w}^*) = [L_{n}({\bf v}_n^t, {\bf w}^t) - L_{n}({\bf v}_n^*, {\bf w}^t)] + [L_{n}({\bf v}_n^*, {\bf w}^t)-L_{n}({\bf v}_n^*, {\bf w}^*)]$. We now bound each of the two terms. 
First, based on the first part of {\bf Lemma}~\ref{lem:cvg_onevariable}, $L_{n}({\bf v}_n^t, {\bf w}^t) - L_{n}({\bf v}_n^*, {\bf w}^t) \leq \Lambda^t \big(L_{n}({\bf v}_n^0, {\bf w}^t) - L_{n}({\bf v}_n^*, {\bf w}^t) \big) = \Lambda^t \big(\tilde{R}(\lambda,\tau; {\bf v}_n^0) - \tilde{R}(\lambda,\tau; {\bf v}_n^*) + \frac{\mu}{2}\|{\bf v}_n^0 - {\bf w}^t \|^2 - \frac{\mu}{2}\|{\bf v}_n^* - {\bf w}^t \|^2 \big) 
\leq \Lambda^t \big( (\tilde{R}(\lambda,\tau; {\bf v}_n^0) - \tilde{R}(\lambda,\tau; {\bf v}_n^*)) + \frac{\mu}{2} \big( \|{\bf v}_n^0 - {\bf w}^* \|^2 + \|{\bf w}^* - {\bf w}^t \|^2 +  \|{\bf w}^t\|^2 - \|{\bf v}_n^*\|^2 \big) \big) \leq \Lambda^t \big( D + \frac{\mu}{2} g(t) \big)$, where $(\tilde{R}(\lambda,\tau; {\bf v}_n^0) - \tilde{R}(\lambda,\tau; {\bf v}_n^*)) + \frac{\mu}{2} \big( \|{\bf v}_n^0 - {\bf w}^* \|^2 + \|{\bf w}^t\|^2 - \|{\bf v}_n^*\|^2 \leq D$. 
For the second term, based on the second part of {\bf Lemma}~\ref{lem:cvg_onevariable}, we have $L_{n}({\bf v}_n^*, {\bf w}^t)-L_{n}({\bf v}_n^*, {\bf w}^*) \leq \Gamma^t \big( L_{n}({\bf v}_n^*, {\bf w}^0)-L_{n}({\bf v}_n^*, {\bf w}^*)  \big) =  \Gamma^t \big( \| {\bf v}_n^* - {\bf w}^0 \|^2 - \|{\bf v}_n^* - {\bf w}^* \|^2 \big) \leq \Gamma^t \cdot E$, where $\| {\bf v}_n^* - {\bf w}^0 \|^2 - \|{\bf v}_n^* - {\bf w}^* \|^2 \leq E$.  
Hence, $L_{n}({\bf v}_n^t, {\bf w}^t) - L_{n}({\bf v}_n^*, {\bf w}^*) \leq (D+\frac{\mu}{2} g(t)) \Lambda^{t}  +  E \Gamma^t$ and it becomes $0$ when $t \rightarrow \infty$ as $\Lambda, \Gamma < 1$. 
\end{proof}

Finally, we show the convergence result of FedTilt. We first state two assumptions also used in the existing works, e.g., Ditto~\cite{li2021ditto}. 

\noindent {\bf Assumption 2.}
The global model converges with rate $g(t)$. That is, there exists $g(t)$ such that $\lim_{t\rightarrow \infty } g(t)=0, \|{\bf w}^t - {\bf w}^*\|^2 \leq g(t)$. 
E.g., the global model for FedAvg converges with rate $O(1/t)$~\cite{li2019convergence}.

\noindent {\bf Assumption 3.}
The distance between the optimal (initial) client models (i.e., ${\bf v}_n^*, {\bf v}_n^0$) and the optimal (initial) global model (i.e., ${\bf w}^*, {\bf w}^0$) are bounded and ${\bf w}^t, \forall t$ is also norm-bounded.
\begin{theorem}[Convergence result on the client models]
Under Lemma~\ref{lem:cvg_onevariable} and Assumptions 2 and 3,  
for any $\tau,\lambda>0$, after $t$ iterations of gradient descent with the step size $\alpha$ and $\beta$, 
$L_{n}({\bf v}_n^t, {\bf w}^t) - L_{n}({\bf v}_n^*, {\bf w}^*) \leq (D+\frac{\mu}{2} g(t)) \Lambda^{t}  +  E \Gamma^t$, where $\Lambda=(1-\frac{\beta_{\min}+\mu}{B_1+\tau B_2 + \lambda C_3})$, $\Gamma = (1-\frac{\mu}{C_1+\tau C_2 + \lambda C_3})$ and 
$D$ and $E$ are constants defined hereafter. 
\end{theorem}
\begin{proof}
$L_{n}({\bf v}_n^t, {\bf w}^t) - L_{n}({\bf v}_n^*, {\bf w}^*) = [L_{n}({\bf v}_n^t, {\bf w}^t) - L_{n}({\bf v}_n^*, {\bf w}^t)] + [L_{n}({\bf v}_n^*, {\bf w}^t)-L_{n}({\bf v}_n^*, {\bf w}^*)]$. We now bound each of the two terms. 
First, based on the first part of {\bf Lemma}~\ref{lem:cvg_onevariable}, $L_{n}({\bf v}_n^t, {\bf w}^t) - L_{n}({\bf v}_n^*, {\bf w}^t) \leq \Lambda^t \big(L_{n}({\bf v}_n^0, {\bf w}^t) - L_{n}({\bf v}_n^*, {\bf w}^t) \big) = \Lambda^t \big(\tilde{R}(\lambda,\tau; {\bf v}_n^0) - \tilde{R}(\lambda,\tau; {\bf v}_n^*) + \frac{\mu}{2}\|{\bf v}_n^0 - {\bf w}^t \|^2 - \frac{\mu}{2}\|{\bf v}_n^* - {\bf w}^t \|^2 \big) 
\leq \Lambda^t \big( (\tilde{R}(\lambda,\tau; {\bf v}_n^0) - \tilde{R}(\lambda,\tau; {\bf v}_n^*)) + \frac{\mu}{2} \big( \|{\bf v}_n^0 - {\bf w}^* \|^2 + \|{\bf w}^* - {\bf w}^t \|^2 +  \|{\bf w}^t\|^2 - \|{\bf v}_n^*\|^2 \big) \big) \leq \Lambda^t \big( D + \frac{\mu}{2} g(t) \big)$, where $(\tilde{R}(\lambda,\tau; {\bf v}_n^0) - \tilde{R}(\lambda,\tau; {\bf v}_n^*)) + \frac{\mu}{2} \big( \|{\bf v}_n^0 - {\bf w}^* \|^2 + \|{\bf w}^t\|^2 - \|{\bf v}_n^*\|^2 \leq D$. 
For the second term, based on the second part of {\bf Lemma}~\ref{lem:cvg_onevariable}, we have $L_{n}({\bf v}_n^*, {\bf w}^t)-L_{n}({\bf v}_n^*, {\bf w}^*) \leq \Gamma^t \big( L_{n}({\bf v}_n^*, {\bf w}^0)-L_{n}({\bf v}_n^*, {\bf w}^*)  \big) =  \Gamma^t \big( \| {\bf v}_n^* - {\bf w}^0 \|^2 - \|{\bf v}_n^* - {\bf w}^* \|^2 \big) \leq \Gamma^t \cdot E$, where $\| {\bf v}_n^* - {\bf w}^0 \|^2 - \|{\bf v}_n^* - {\bf w}^* \|^2 \leq E$.  
Hence, $L_{n}({\bf v}_n^t, {\bf w}^t) - L_{n}({\bf v}_n^*, {\bf w}^*) \leq (D+\frac{\mu}{2} g(t)) \Lambda^{t}  +  E \Gamma^t$ and it becomes $0$ when $t \rightarrow \infty$ as $\Lambda, \Gamma < 1$. 
\end{proof}

{\bf Theorem}~\ref{thm:clientcvg} indicates that solving the tilted ERM local objective to a local optimum using the gradient-based method in Algorithm~\ref{app:Algorithm} is as efficient as traditional ERM objective.

\begin{figure*}[!t]
    \centering
    \subfloat[Gaussian noises {\label{mnist_noise1}}]
    {\includegraphics[bb=0 0 580 600, scale=0.14]{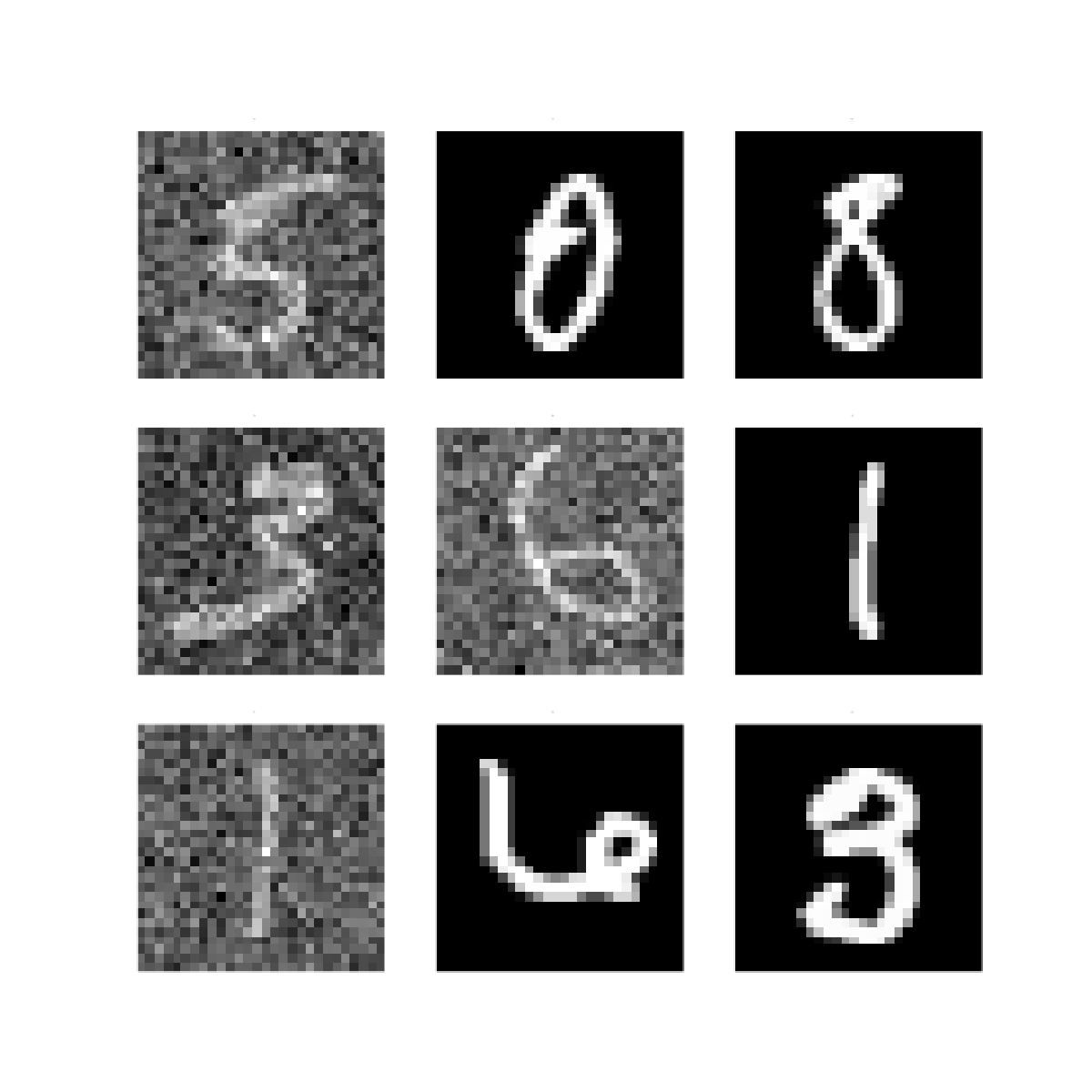}}
    \subfloat[Random corruptions {\label{mnist_noise2}}]
    {\includegraphics[bb=0 0 580 600, scale=0.14]{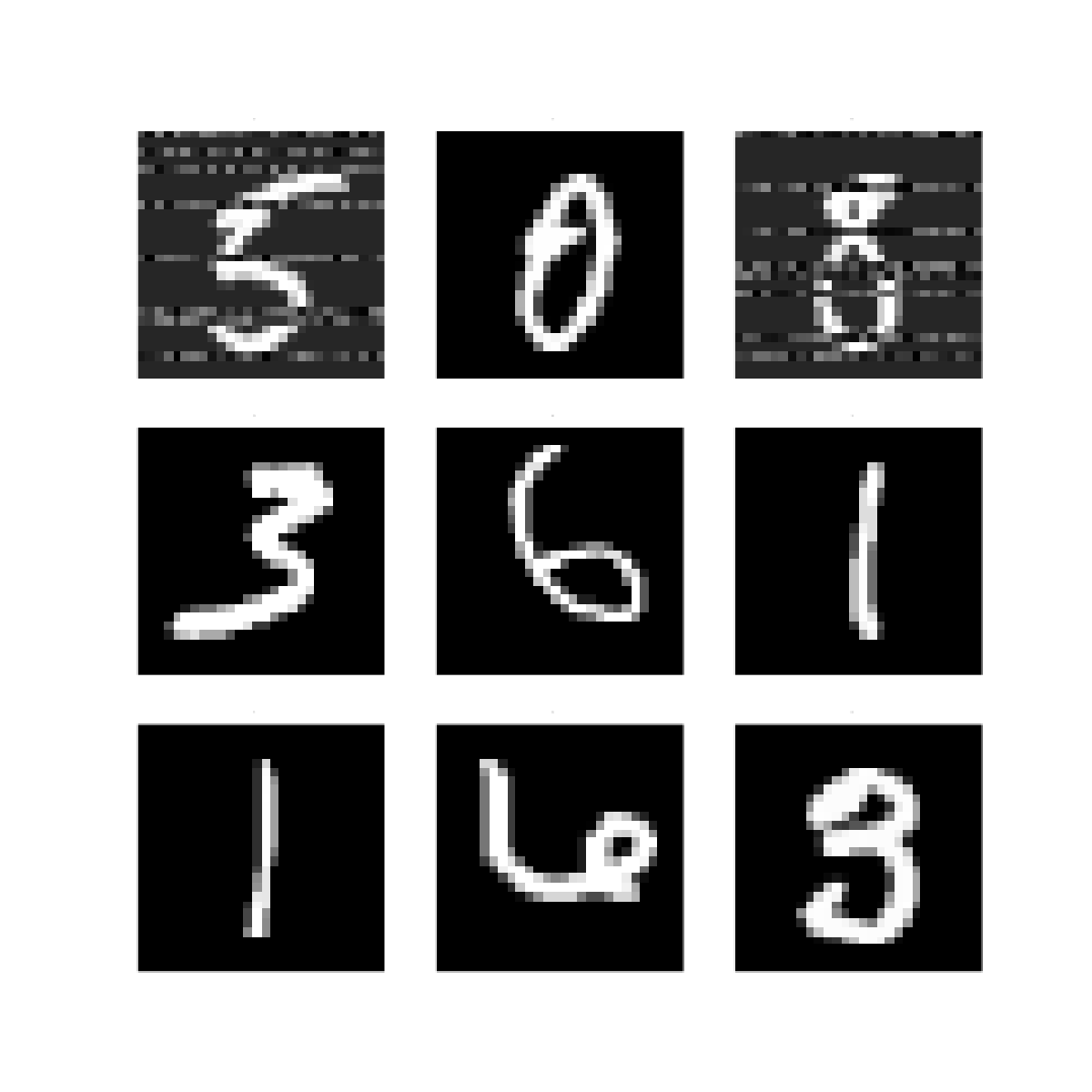}}
    \subfloat[Gaussian noises {\label{fmnist_noise1}}
    ]{\includegraphics[bb=0 0 580 600, scale=0.14]{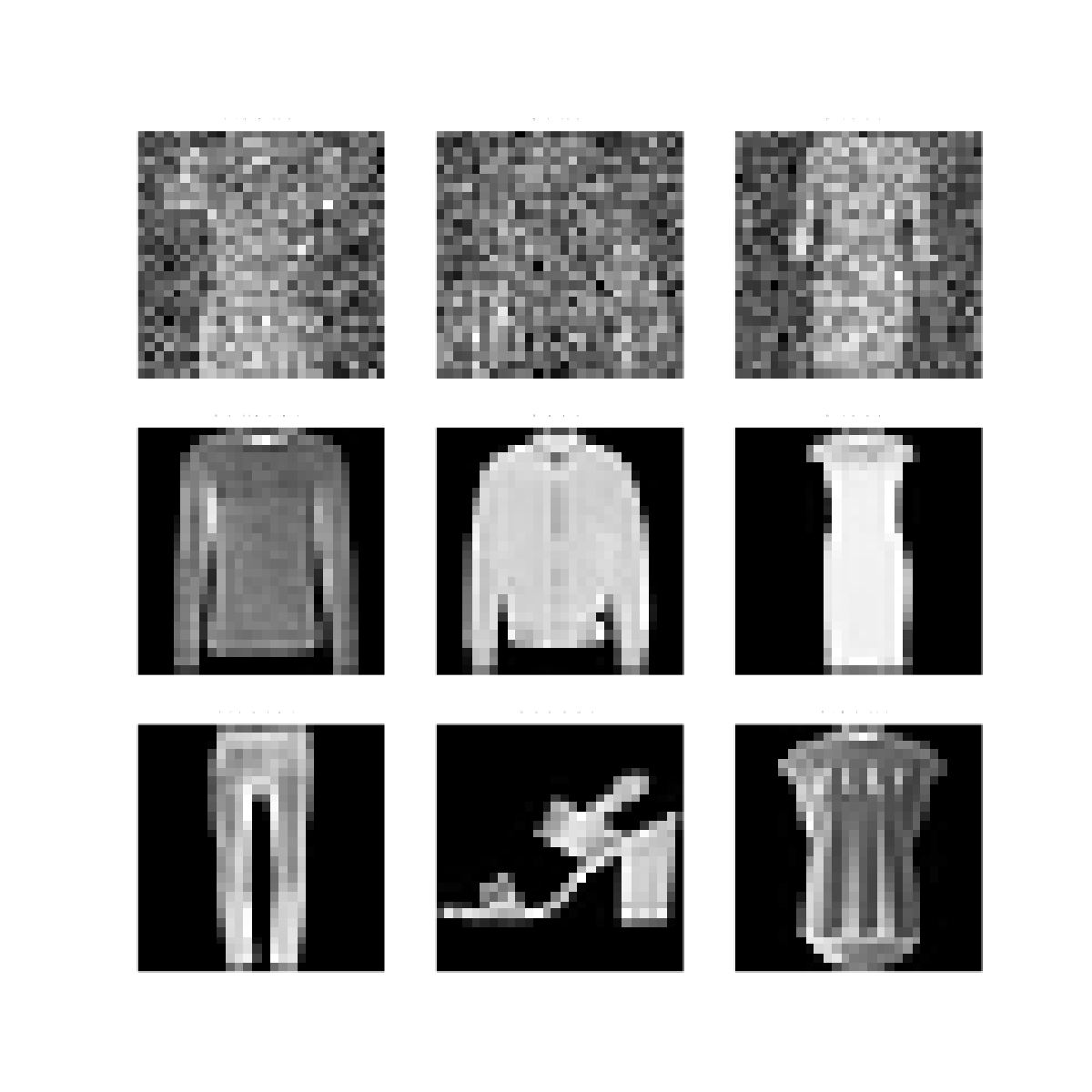}}
    \subfloat[Random corruptions {\label{fmnist_noise2}}
    ]{\includegraphics[bb=0 0 580 600, scale=0.14]{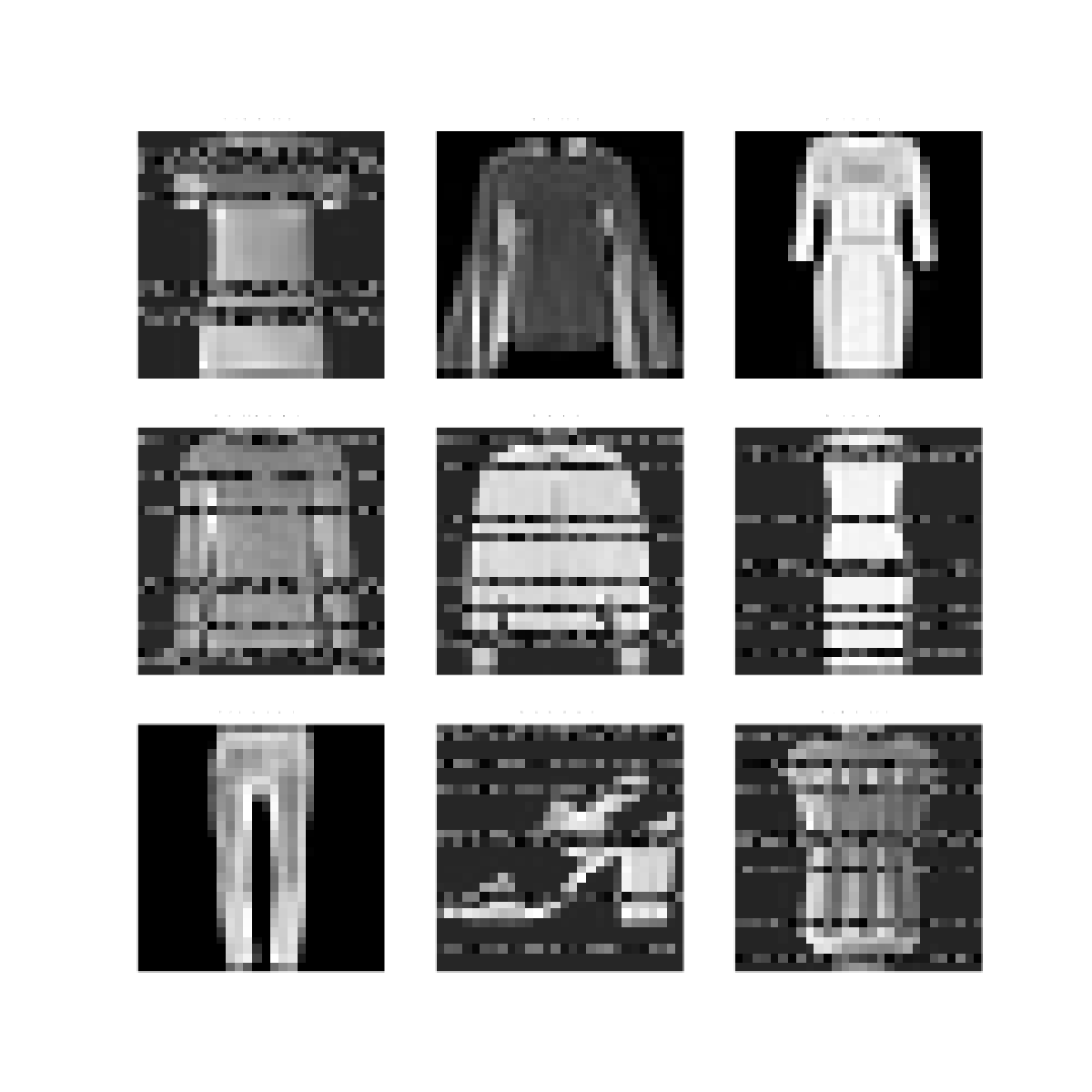}}
    \subfloat[Gaussian noises {\label{cifar_noise1}}
    ]{\includegraphics[bb=0 0 580 600, scale=0.14]{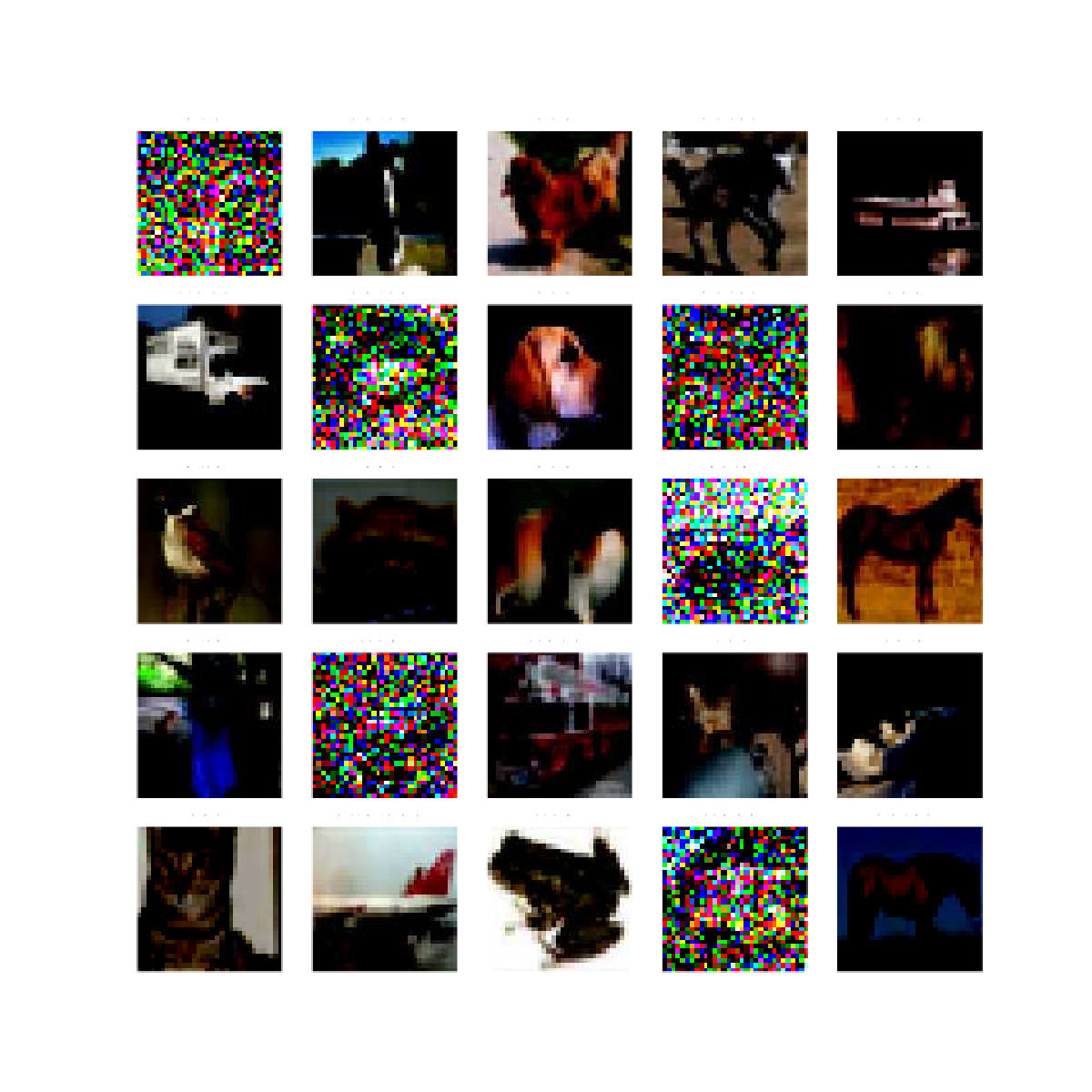}}
    \subfloat[Random corruptions {\label{cifar_noise2}}
    ]{\includegraphics[bb=0 0 580 600, scale=0.14]{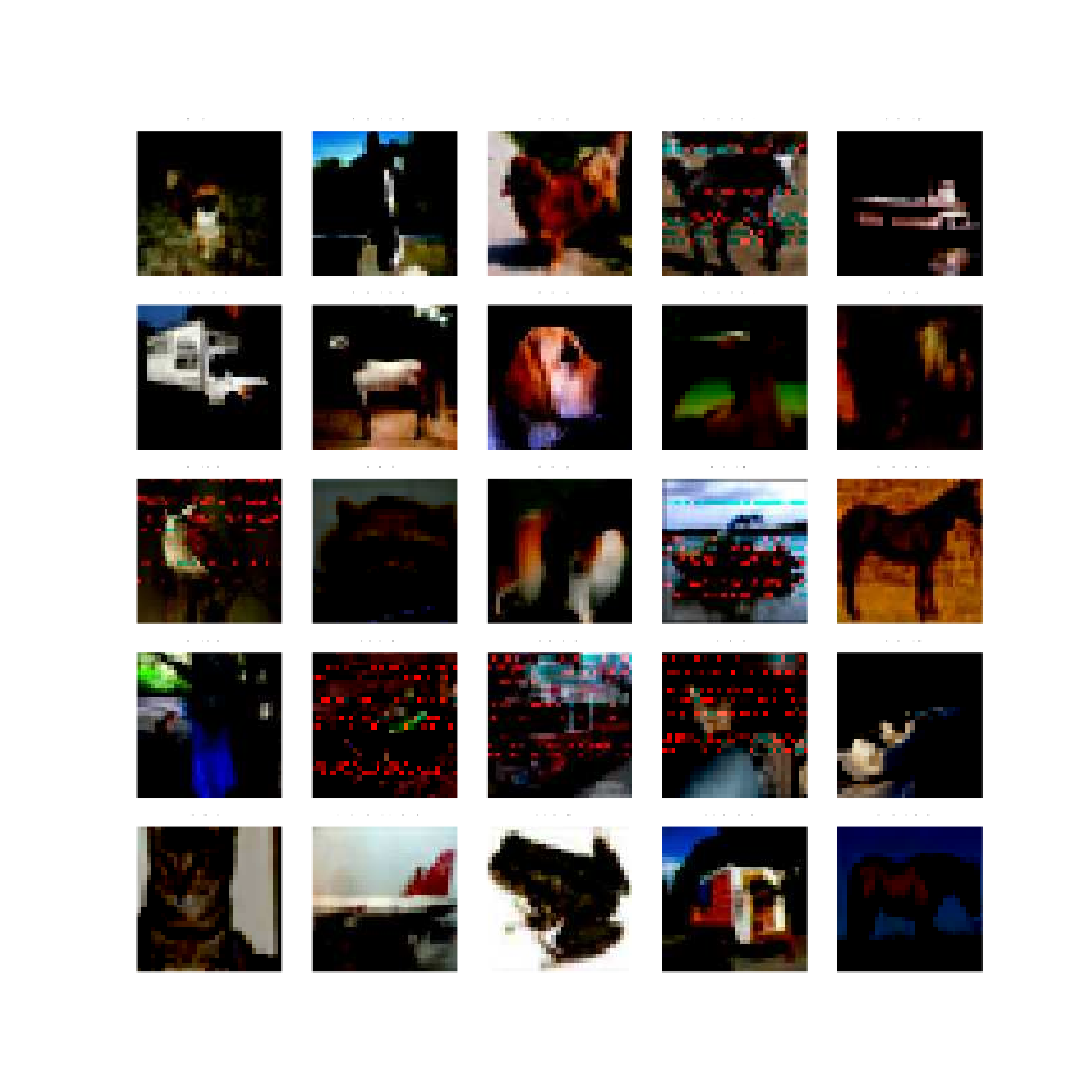}}
    \caption{Example MNIST (a) and (b), FashionMNIST (c) and (d), and CIFAR10 (e) and (f) with outliers. 
    } \label{Outliers}
\end{figure*}

\begin{table}[!ht]
\centering
\caption{Setup of Toy Example Experiments}
\label{tbl:fedtilt_summary}
\begin{tabular}{|c|c|c|c|c|}
\hline
\textbf{Exp} & \textbf{Client} & \textbf{Group} & \textbf{Center} & \textbf{Std Dev}\\
\hline
1 & 1 & 1 & $(0.5, 2.0)$ & $\sigma=0.5$ \\
\hline
1 & 1 & 2 & $(2.5, 1.0)$ & $\sigma=0.5$ \\
\hline
1 & 2 & 1 & $(1.0, 2.2)$ & $\sigma=0.5$ \\
\hline
1 & 2 & 2 & $(2.2, 0.8)$ & $\sigma=0.5$ \\
\hline
2 & 1 & 1 & $(0.5, 2.0)$ & $\sigma=0.35$ \\
\hline
2 & 1 & 2 & $(2.0, 1.0)$ & $\sigma=0.25$ \\
\hline
2 & 2 & 1 & $(0.5, 2.0)$ & $\sigma=0.35$ \\
\hline
2 & 2 & 2 & $(2.5, 1.8)$ & $\sigma=0.25$ \\
\hline
3 & 1 & 1 & $(1.0, 2.0)$ & $\sigma=1.0$ \\
\hline
3 & 1 & 2 & $(2.5, 1.0)$ & $\sigma=0.3$ \\
\hline
3 & 2 & 1 & $(1.0, 2.0)$ & $\sigma=1.0$ \\
\hline
3 & 2 & 2 & $(2.5, 1.0)$ & $\sigma=0.3$ \\
\hline
\end{tabular}
\end{table}

\subsection{More Experiments}
\label{app:exp}

\begin{figure*}[!t]
    \vspace{-0.45in}
    \centering
    \subfloat[\centering Client fairness: Client 1. class ratio=1:1, $\tau=1,\lambda=1$]
    {\includegraphics[bb=0 0 575 575, width=0.3\textwidth]{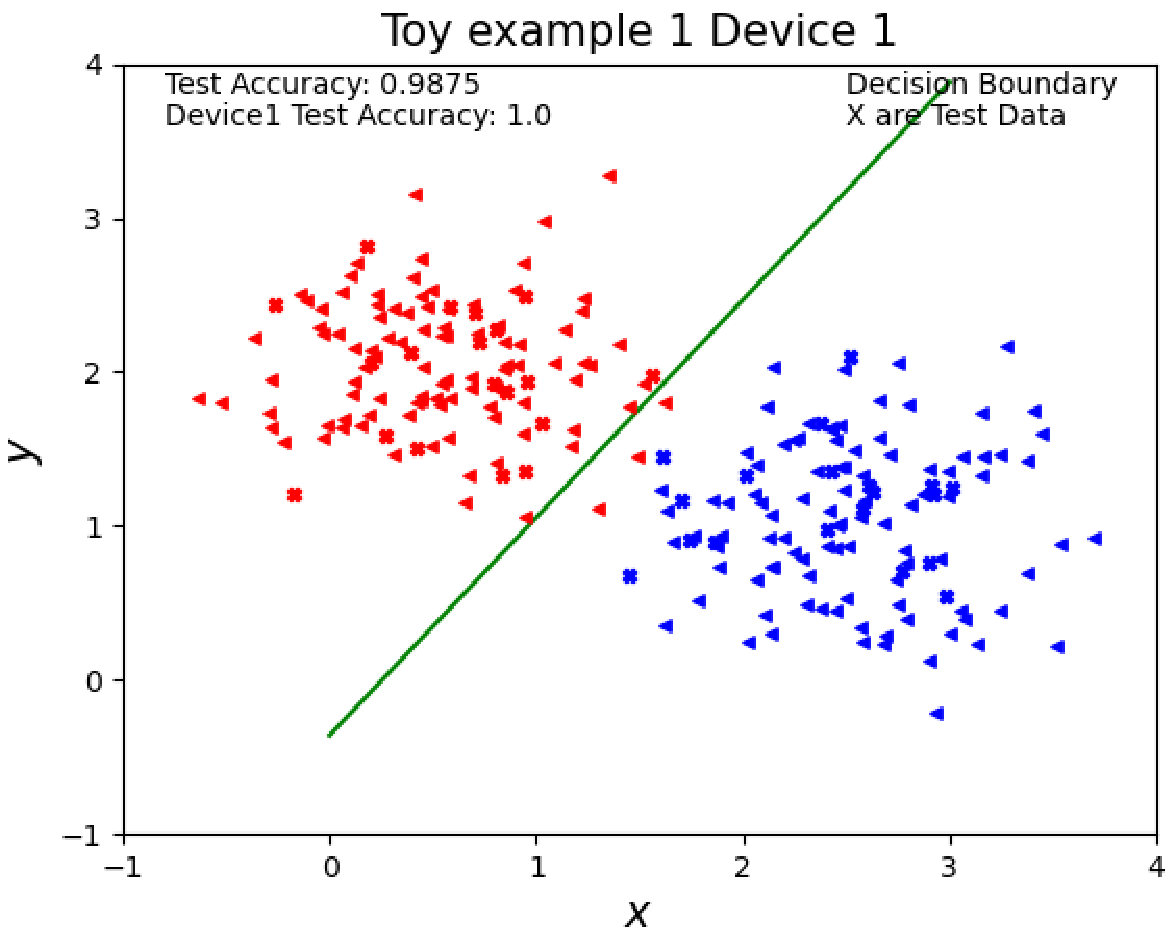}}
    \subfloat[{\centering Two-level fairness: Client 1. class ratio=3:1, $\lambda=10$}]
    {\includegraphics[bb=0 0 575 575, width=0.3\textwidth]{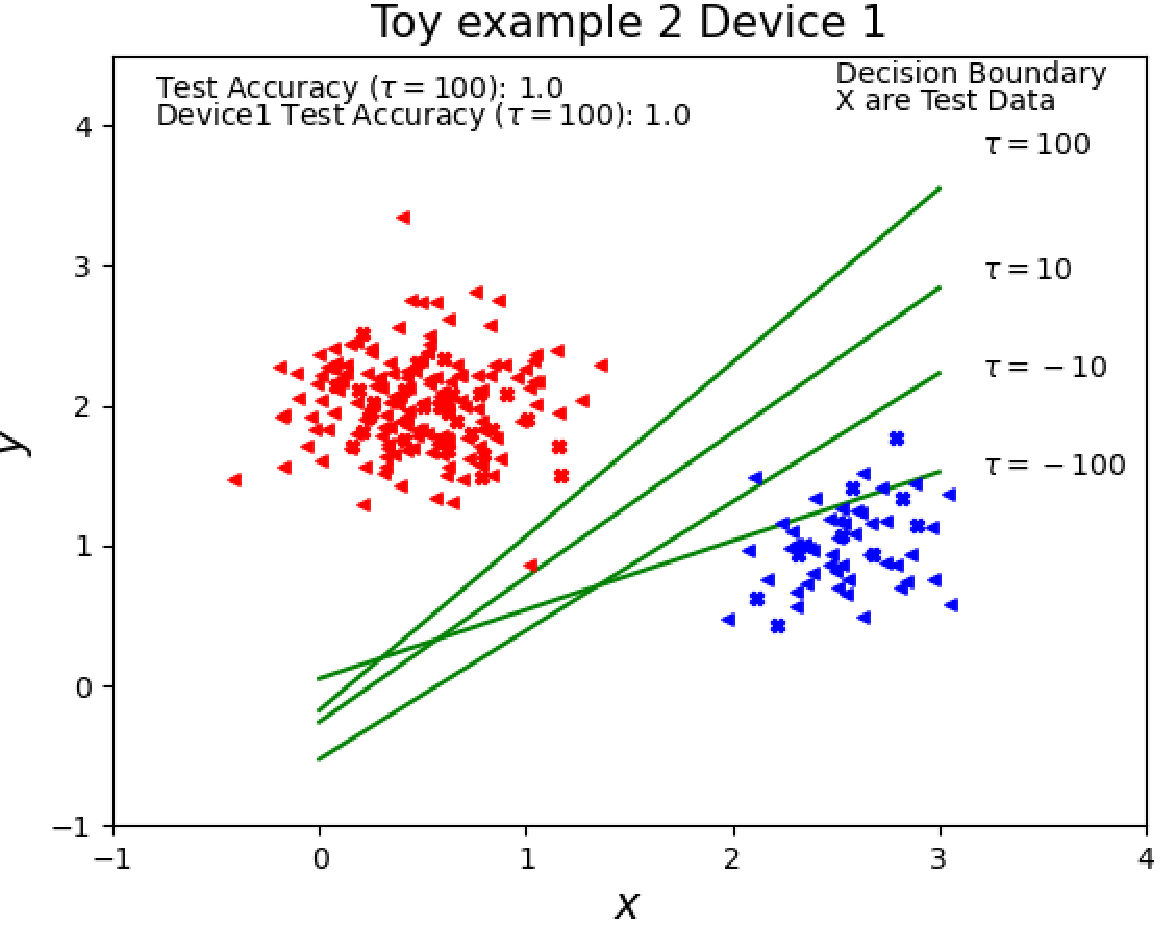}}
    \subfloat[{\centering Two-level fairness and robustness. class ratio=3:1, $\tau=10$}]
    {\includegraphics[bb=0 0 575 575, width=0.3\textwidth]{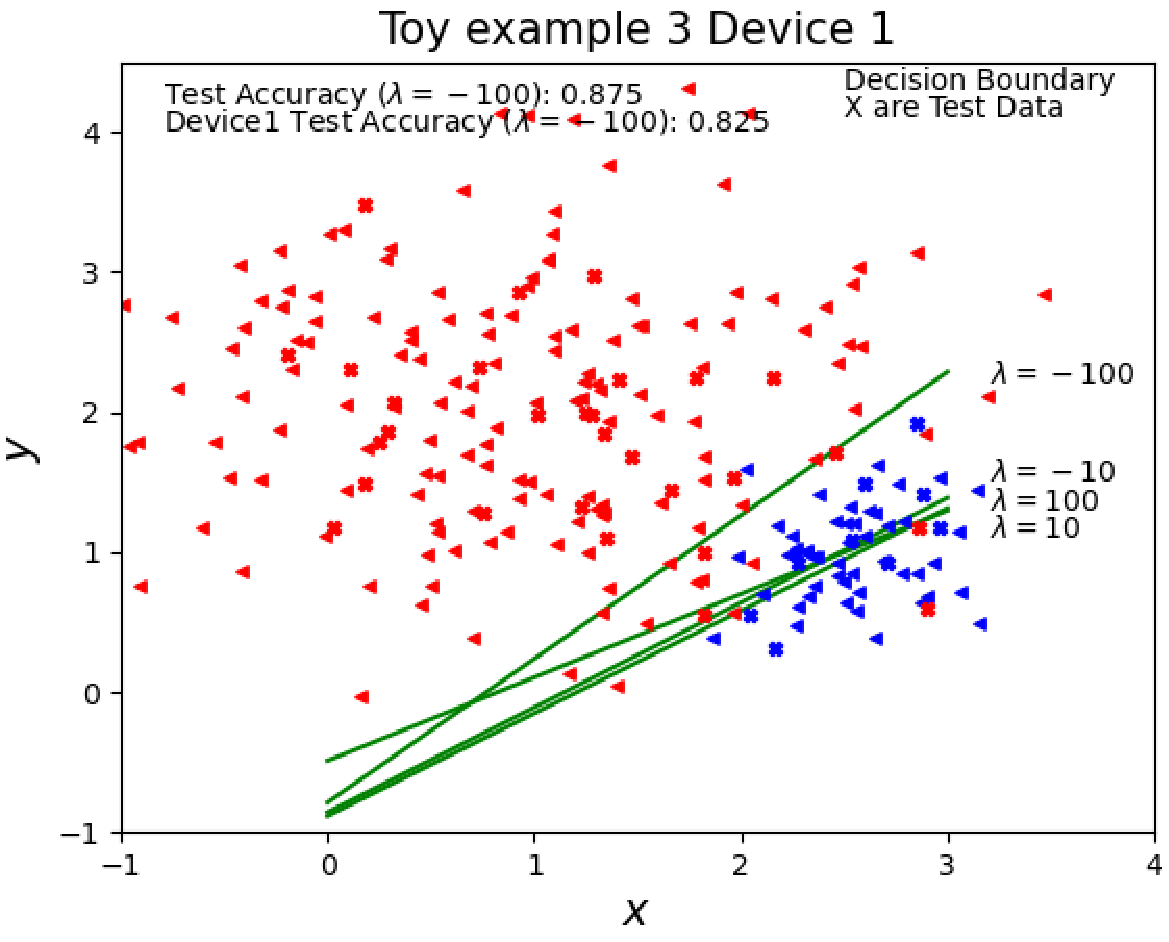}}
    
    \vspace{-0.35in}
    \subfloat[{\centering Client fairness: Client 2. class ratio=1:1, $\tau=1,\lambda=1$}]
    {\includegraphics[bb=0 0 575 575, width=0.3\textwidth]{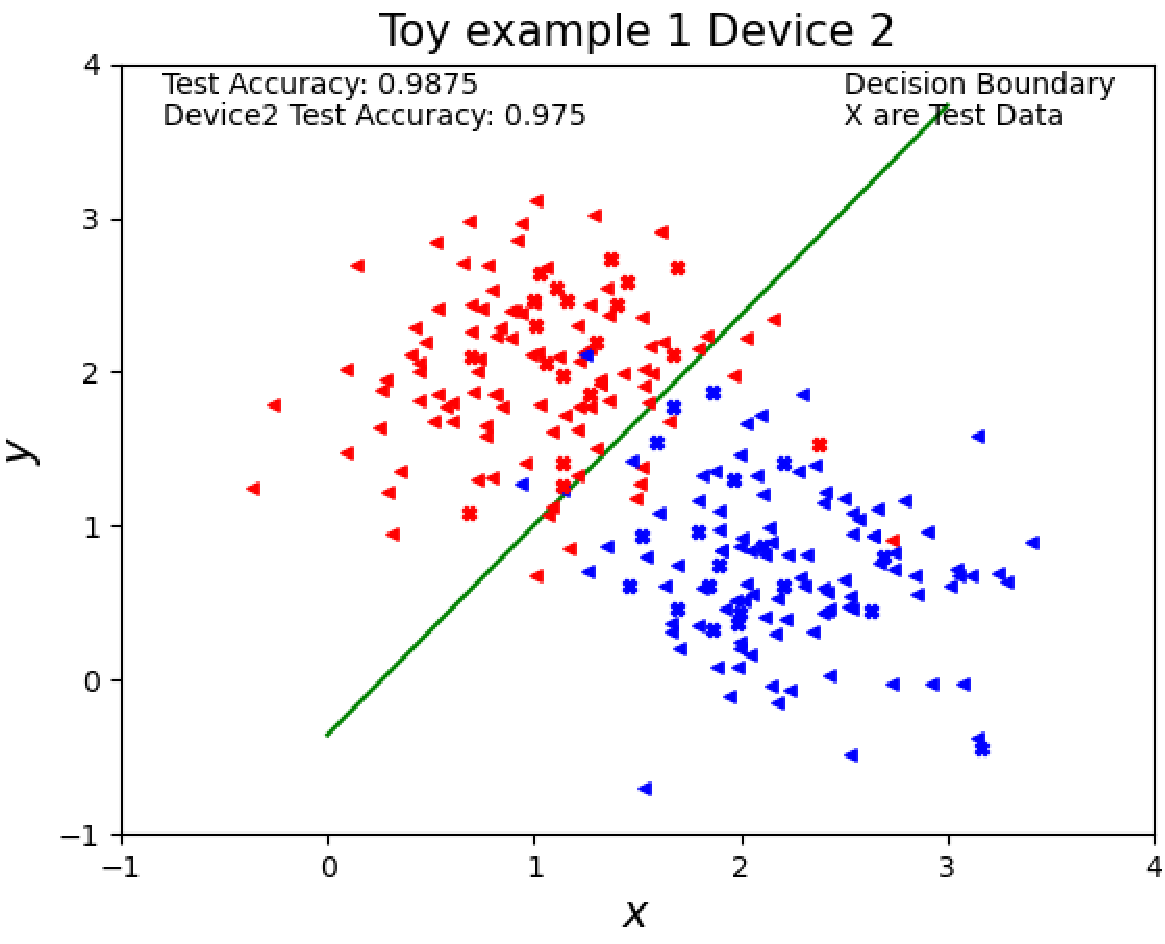}}
    \subfloat[{\centering Two-level fairness: Client 2. class ratio=3:1, $\lambda=10$}]
    {\includegraphics[bb=0 0 575 575, width=0.3\textwidth]{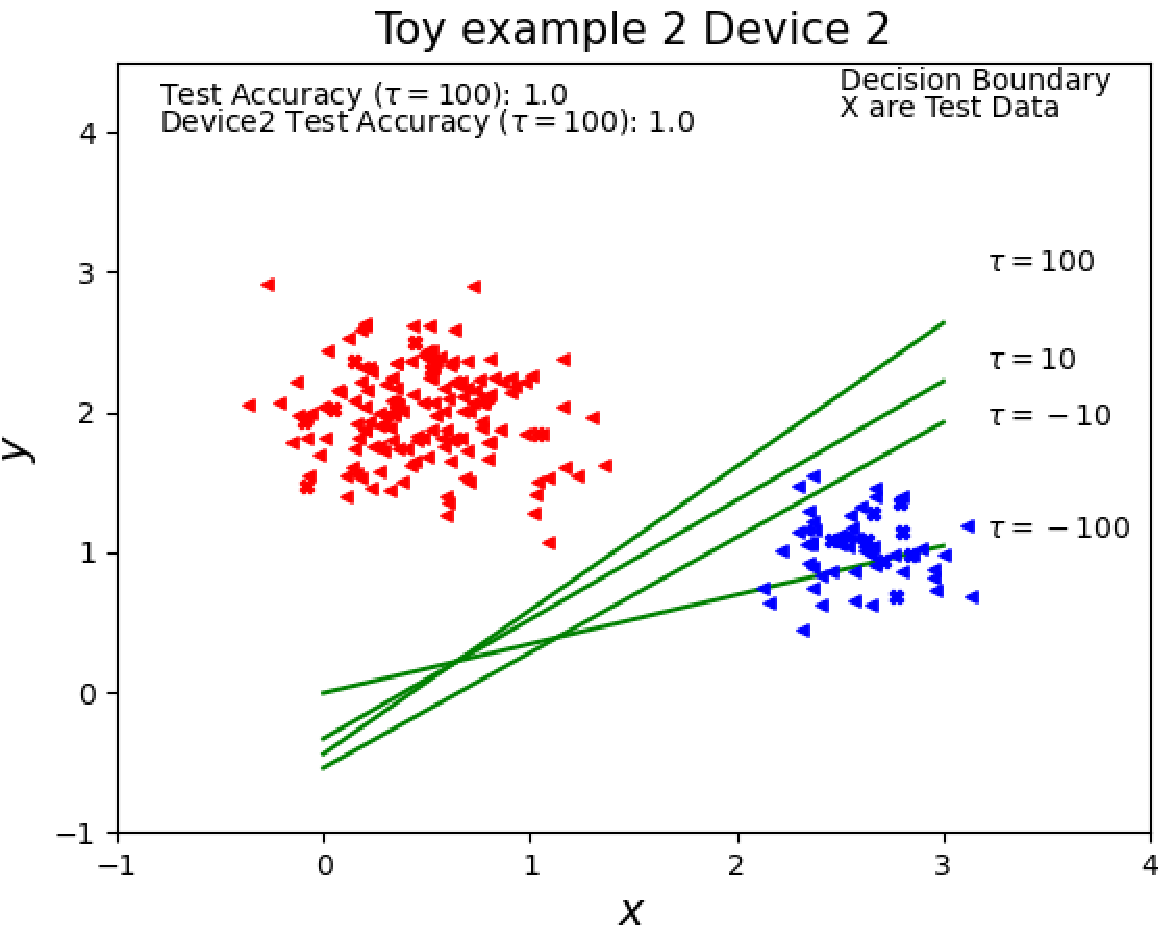}}
    \subfloat[{\centering Two-level fairness and robustness. class ratio=3:1, $\tau=10$}]
    {\includegraphics[bb=0 0 575 575, width=0.3\textwidth]{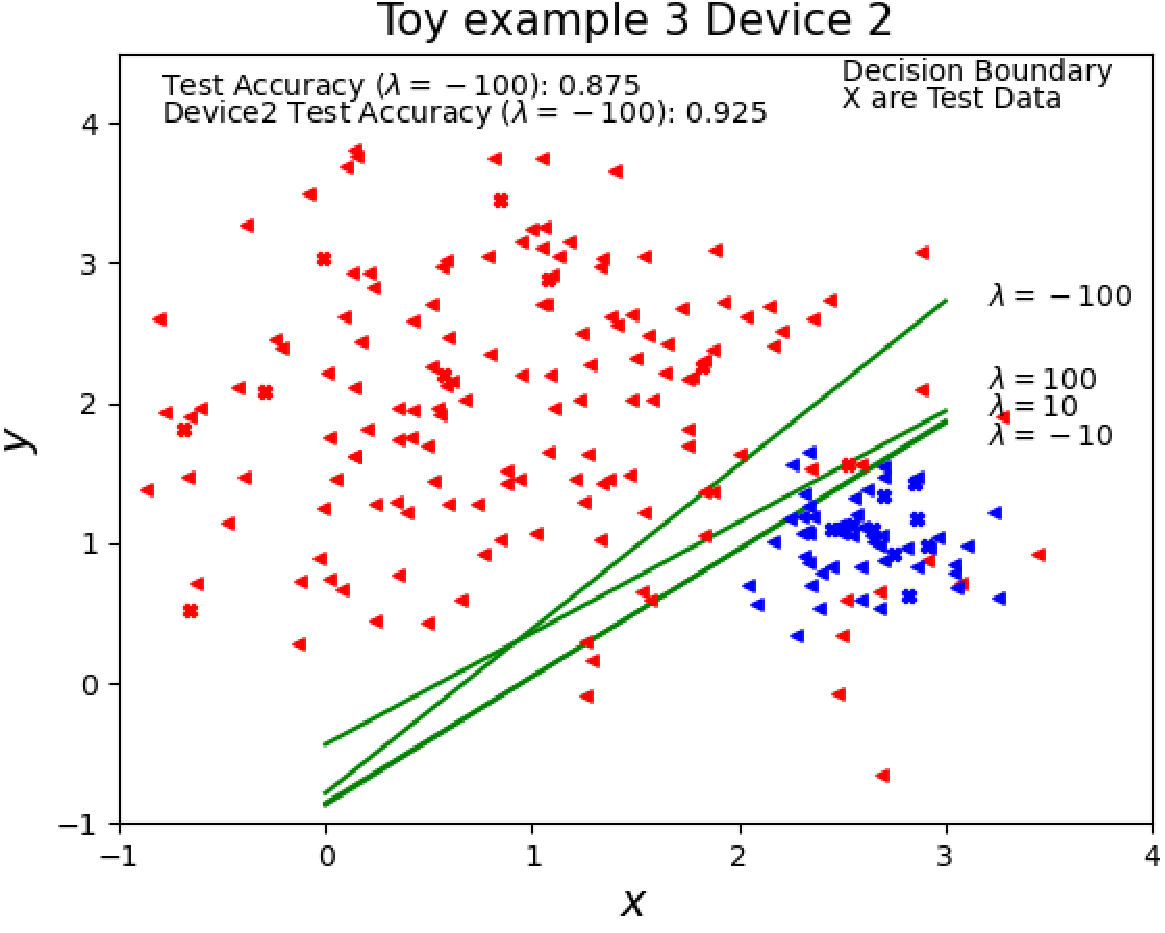}}
    \caption{Federated logistic regression results for binary classification. $q=0$ and \texttt{dist} is Euclidean distance.}
    \label{fig:toyLogReg}
   \vspace{-4mm}
\end{figure*}

\subsubsection{Experimental setup}

\noindent {\bf Datasets and models.} We evaluate FedTilt on three image datasets: MNIST, FashionMNIST (F-Mnist), and CIFAR10.

The MNIST database \cite{lecun-mnisthandwrittendigit-2010} has a training set of 60,000 examples, and a test set of 10,000 examples. It contains handwritten digits between 0 and 9.
The MNIST image classification task uses a multilayer perceptron (MLP)---3 linear layers and uses a ReLU as the activation function. A softmax function is applied to normalize the output of the network. The input of the model is a flattened 784-dim ($28\times28$) image, and the output is a class label between 0 and 9.

F-MNIST is similar to MNIST and used for benchmarking ML algorithms \cite{xiao2017fashionMNIST}. It shares the same image size, structure of training,  testing splits, MLP model, and number of class.

CIFAR10 dataset contains 50,000 32x32 (low-resolution) color training images and 10,000 test images, labeled over 10 categories, i.e., there are 6,000 images of each class. The 10 different classes represent airplanes, cars, birds, cats, deer, dogs, frogs, horses, ships, and trucks. A CNN is used to perform the classification task. The CNN is made of 3 convolutional blocks and a fully connected (FC) layer. 
All layers use ReLU as the activation function. The output of the model is a class label between 0 and 9.

Example clean images and their outliers are shown in Figure~\ref{Outliers}.

\begin{figure}[!t]
    \centering
    \subfloat[{Test Acc + Client fairness\label{mnist_20c_cf}}
    ]{\includegraphics[bb=0 0 1200 600, scale=0.098]{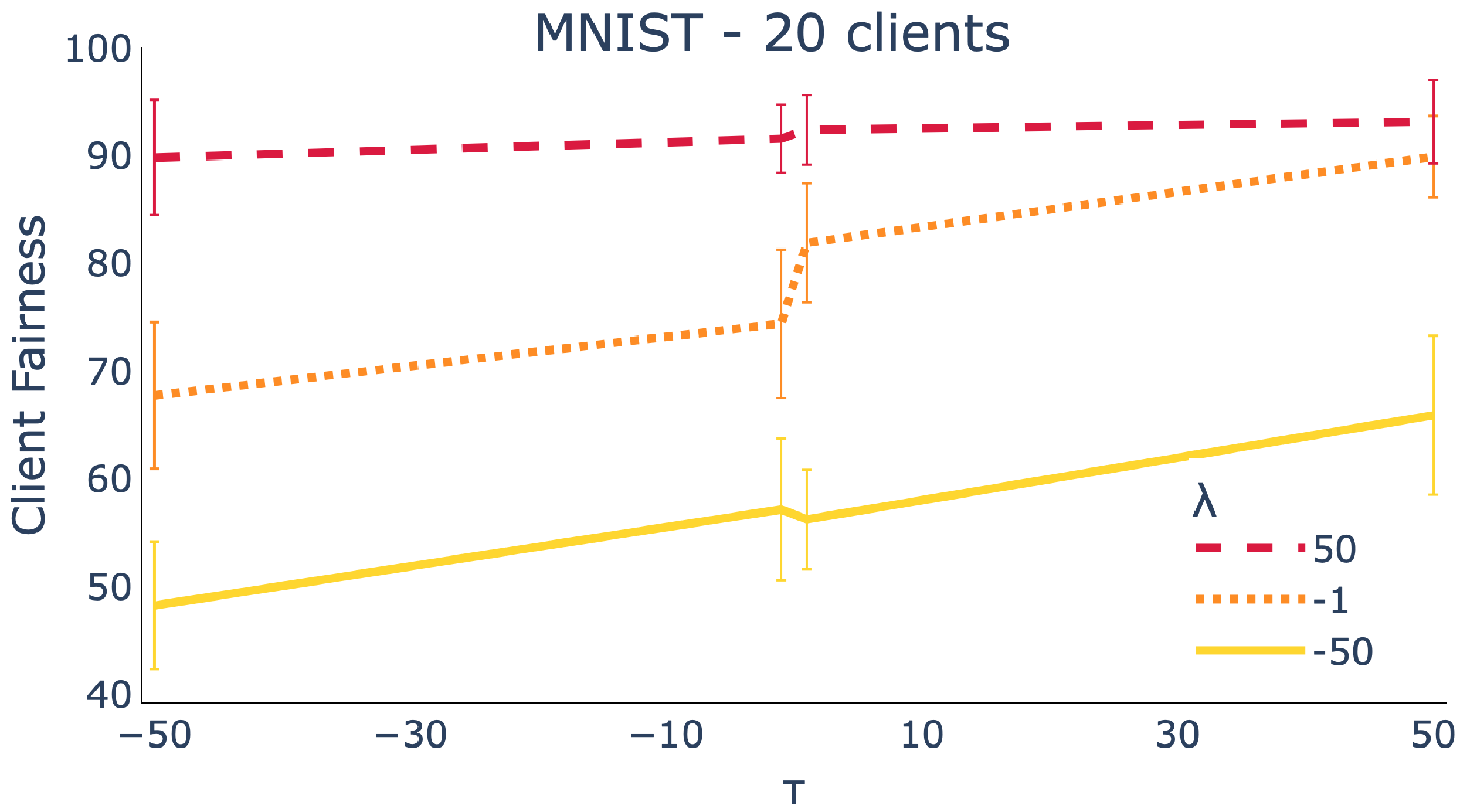}}
    \subfloat[{Client data fairness\label{mnist_20c_cdf}}
    ]{\includegraphics[bb=0 0 1200 600, scale=0.098]{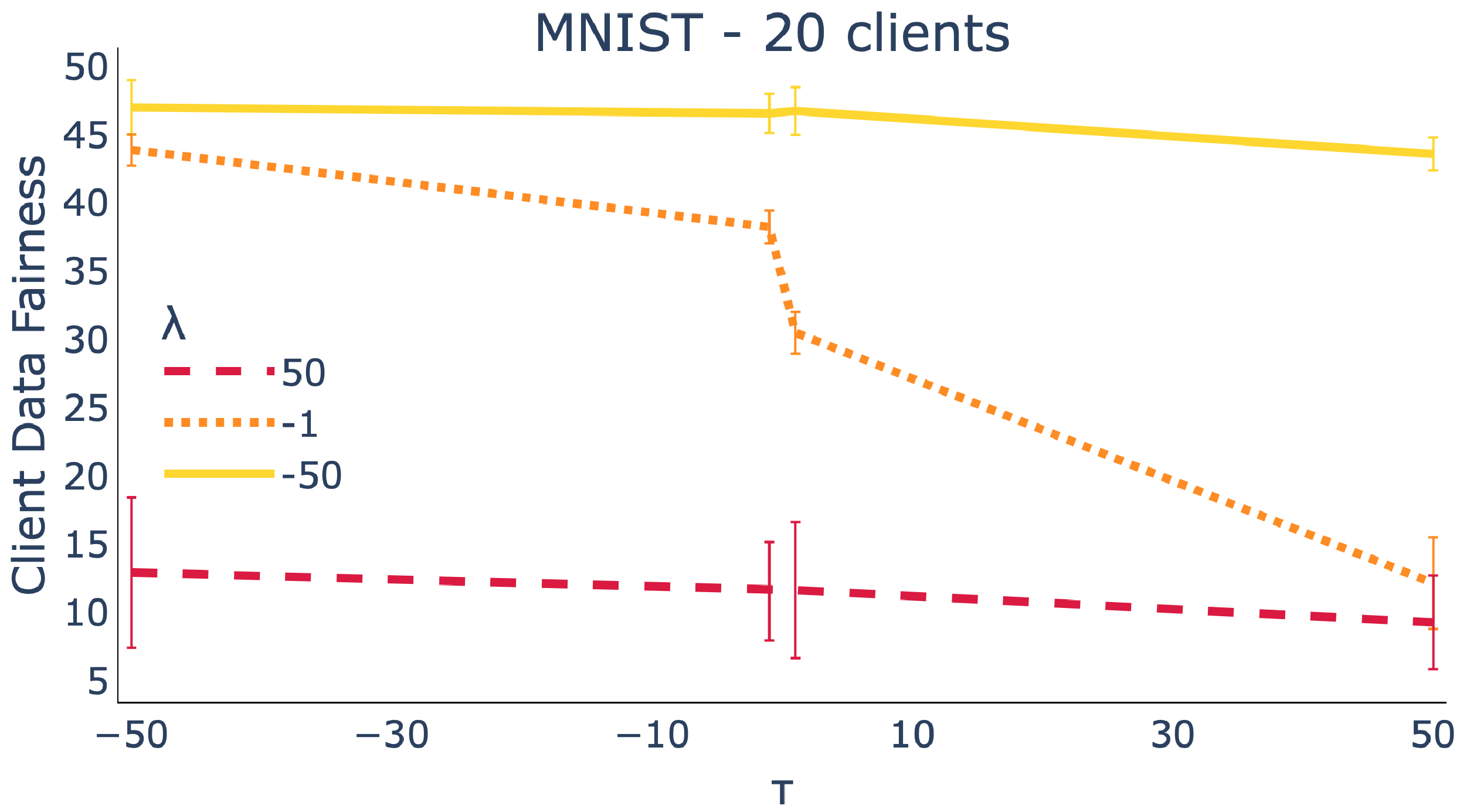}}
    \\
    \vspace{0.2in}
    \subfloat[{Test Acc + Client fairness\label{fmnist_20c_cf}}
    ]{\includegraphics[bb=0 0 1200 600, scale=0.098]{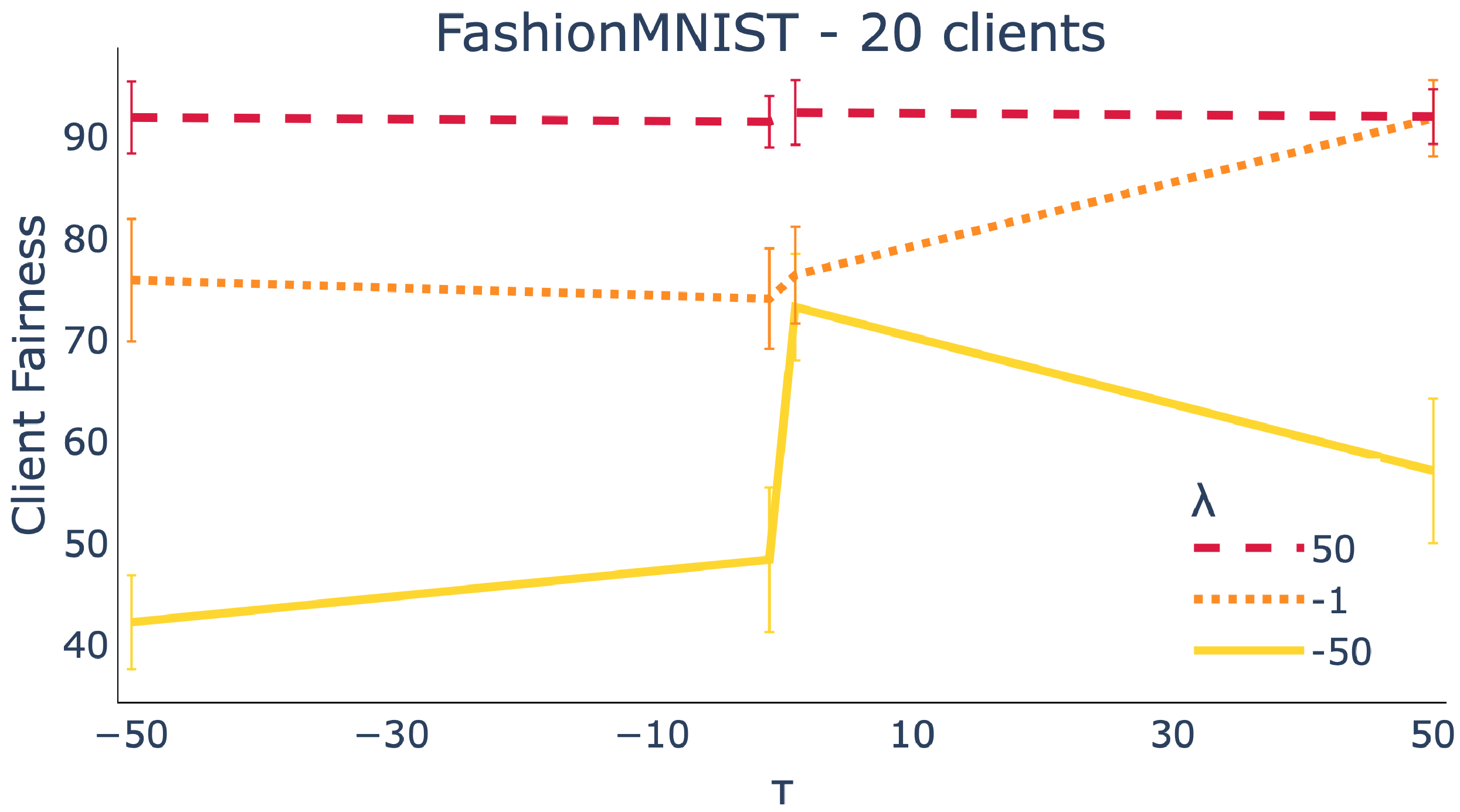}}
    \subfloat[{Client data fairness\label{fmnist_20c_cdf}}
    ]{\includegraphics[bb=0 0 1200 600, scale=0.098]{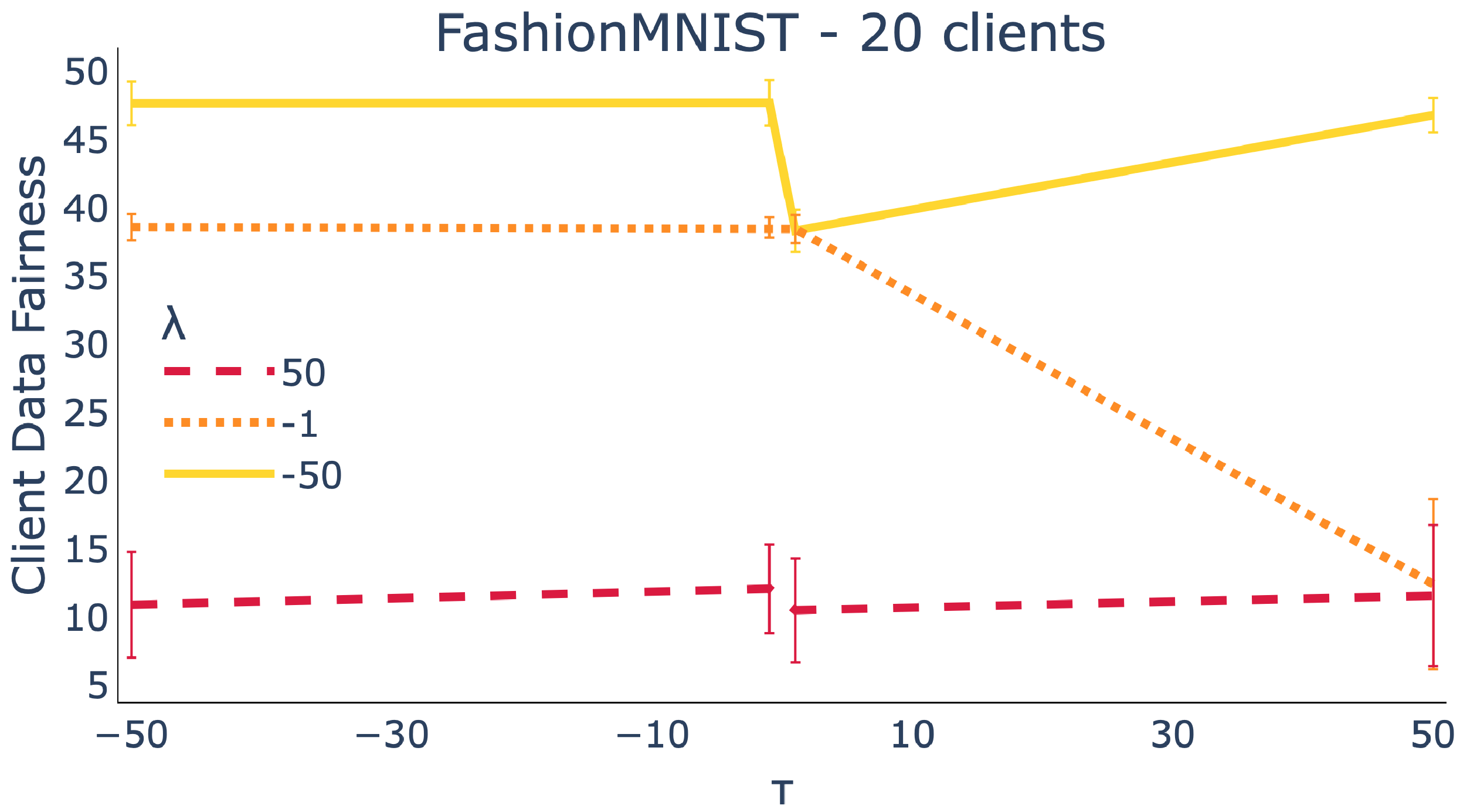}}
    \\
    \vspace{0.2in}
    \subfloat[{Test Acc +  Client fairness\label{cifar_20c_cf}}
    ]{\includegraphics[bb=0 0 1200 600, scale=0.098]{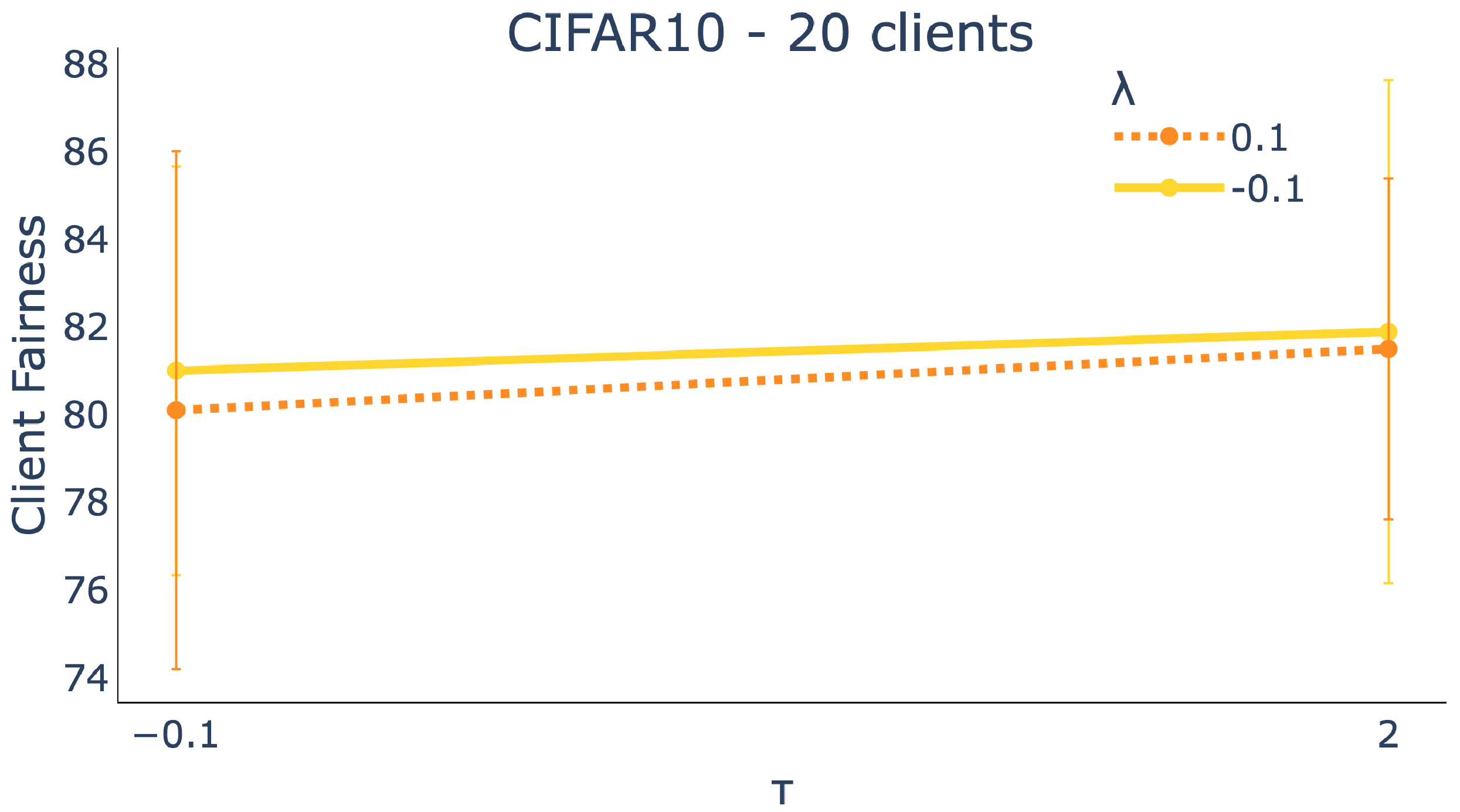}}
    \subfloat[{Client data fairness\label{cifar_20c_cdf}}
    ]{\includegraphics[bb=0 0 1200 600, scale=0.098]{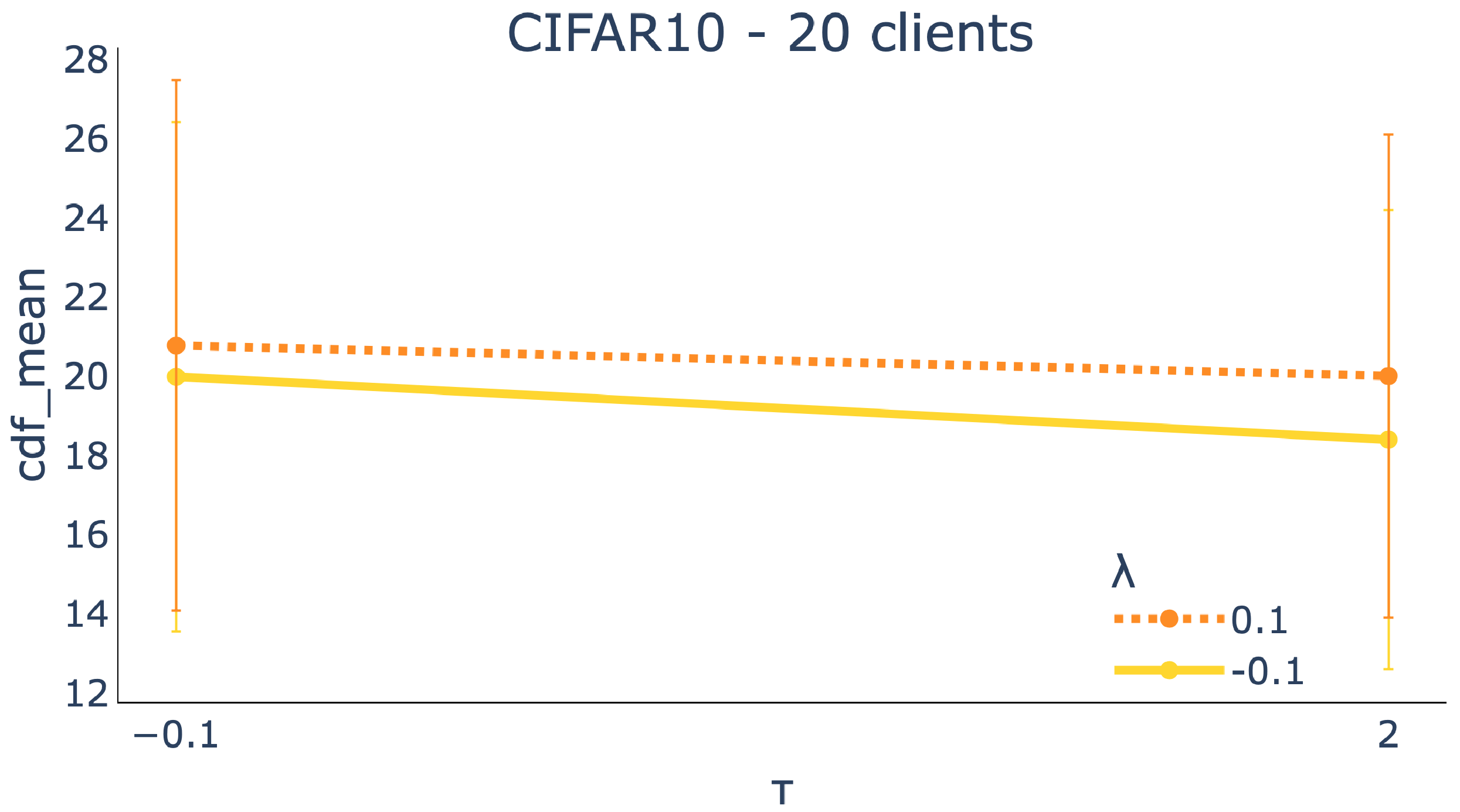}}
    \caption{MNIST results---clean data (a \& b 20 clients). 
    Higher values of both $\lambda$ and $\tau$ provide better results. 
    F-MNIST results---clean data (c \& d 20 clients). 
    Higher values of both $\lambda$ and $\tau$ provide better results.
    CIFAR10 results---clean data (d \& e20 clients). 
    Higher values of both $\lambda$ and $\tau$ provide better results. 
    } \label{mnist_20c}
    \vspace{-4mm}
\end{figure}
\raggedbottom

\begin{figure}[!t]
    \centering
    \subfloat[{Test Acc + Client fairness. 
    \label{mnist_gaussian_cf_tau}}]
    {\includegraphics[bb=0 0 1200 750, scale=0.098]{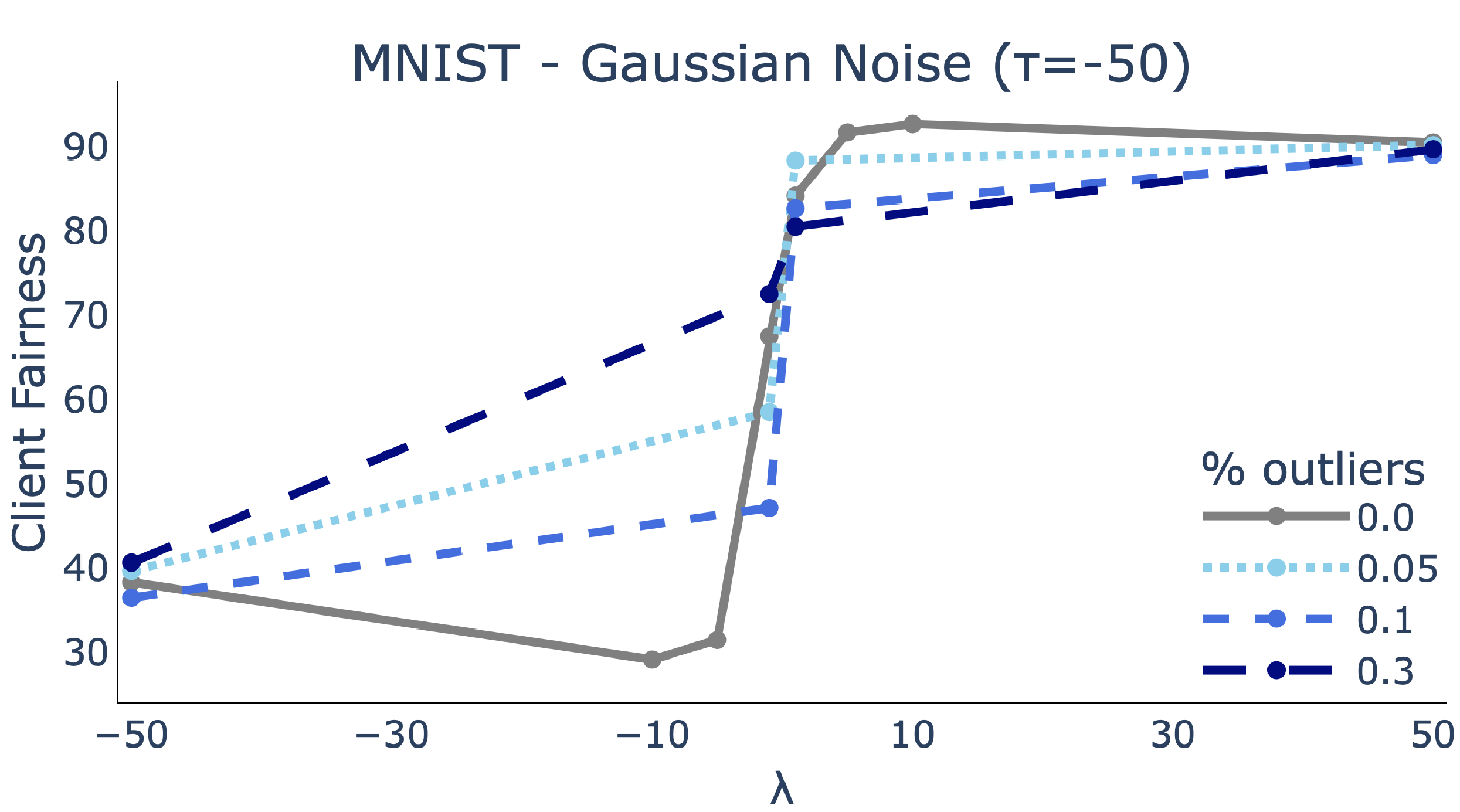}}
    \subfloat[{Client data fairness. 
    \label{mnist_gaussian_cdf_tau}}]
    {\includegraphics[bb=0 0 1200 750, scale=0.098]{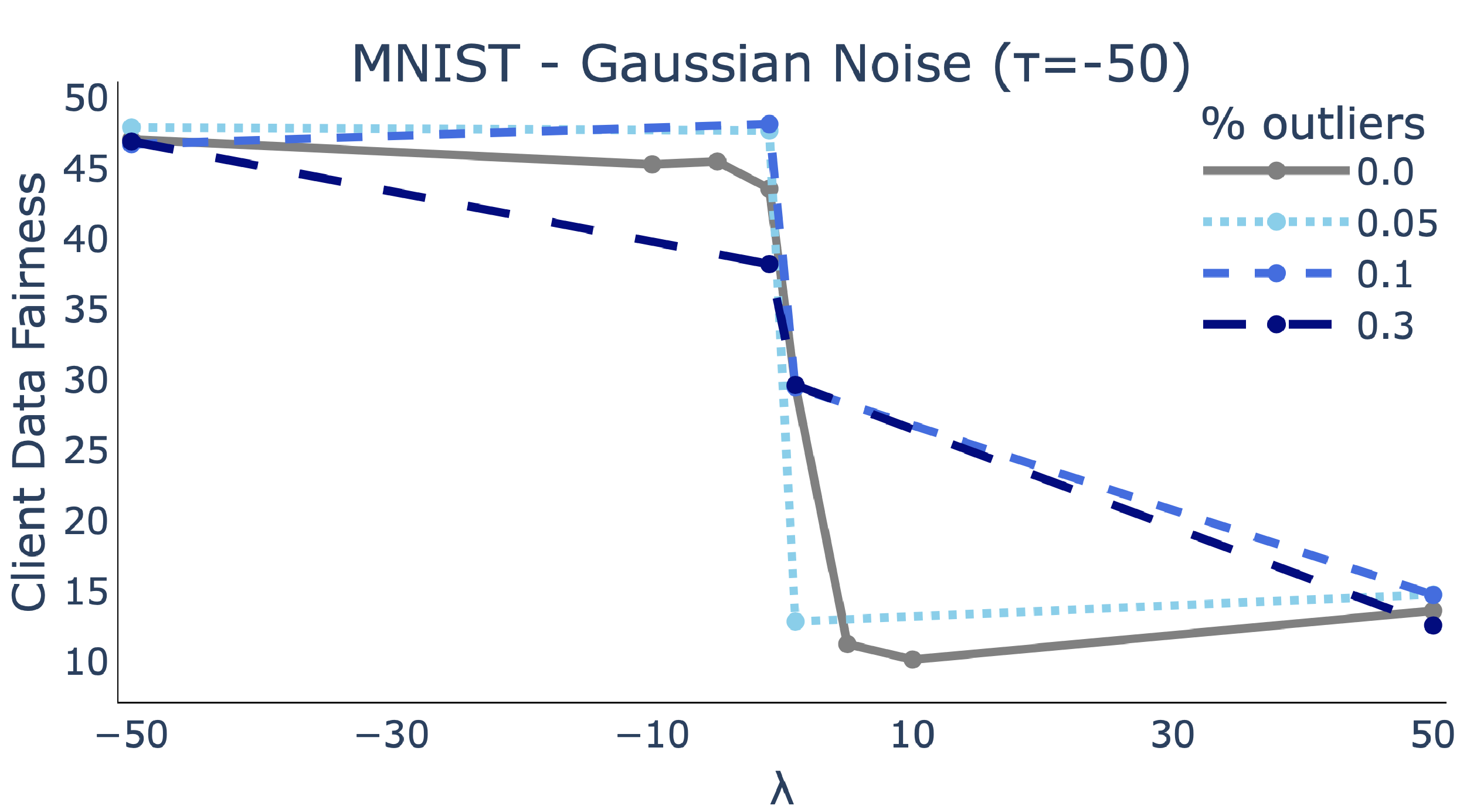}}

    \caption{MNIST results---Gaussian noise. A larger positive $\lambda=50$ and negative $\tau=-50$ show better two-level fairness and robustness results. 
    } \label{mnist_gaussian}
\end{figure}

\begin{figure}[!t]
    \centering
    \subfloat[{Test Acc + Client fairness. 
    \label{fmnist_gaussian_cf_tau}}]
    {\includegraphics[bb=0 0 1200 600, scale=0.098]{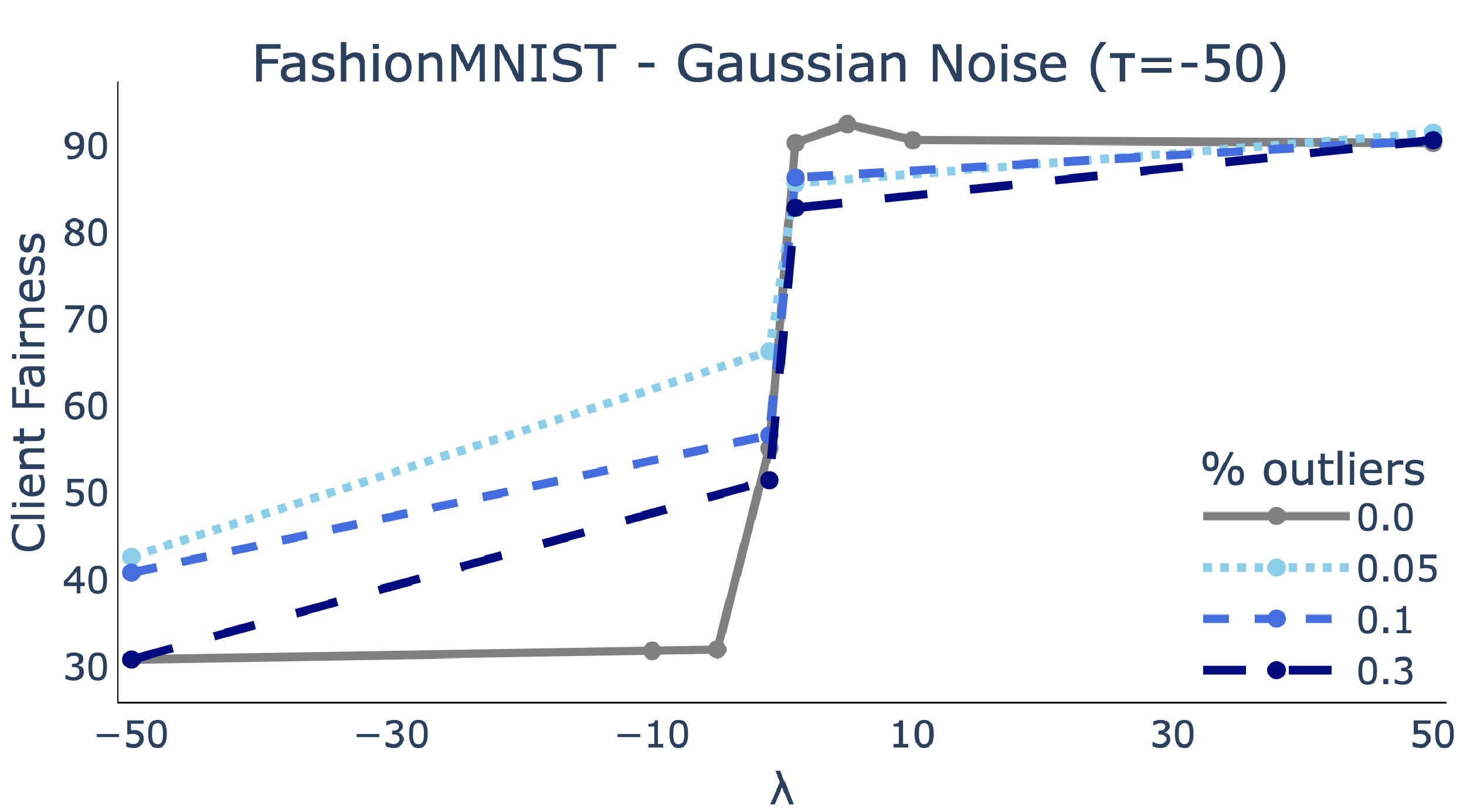}}
    \subfloat[{Client data fairness. 
    \label{fmnist_gaussian_cdf_tau}}]
    {\includegraphics[bb=0 0 1200 600, scale=0.098]{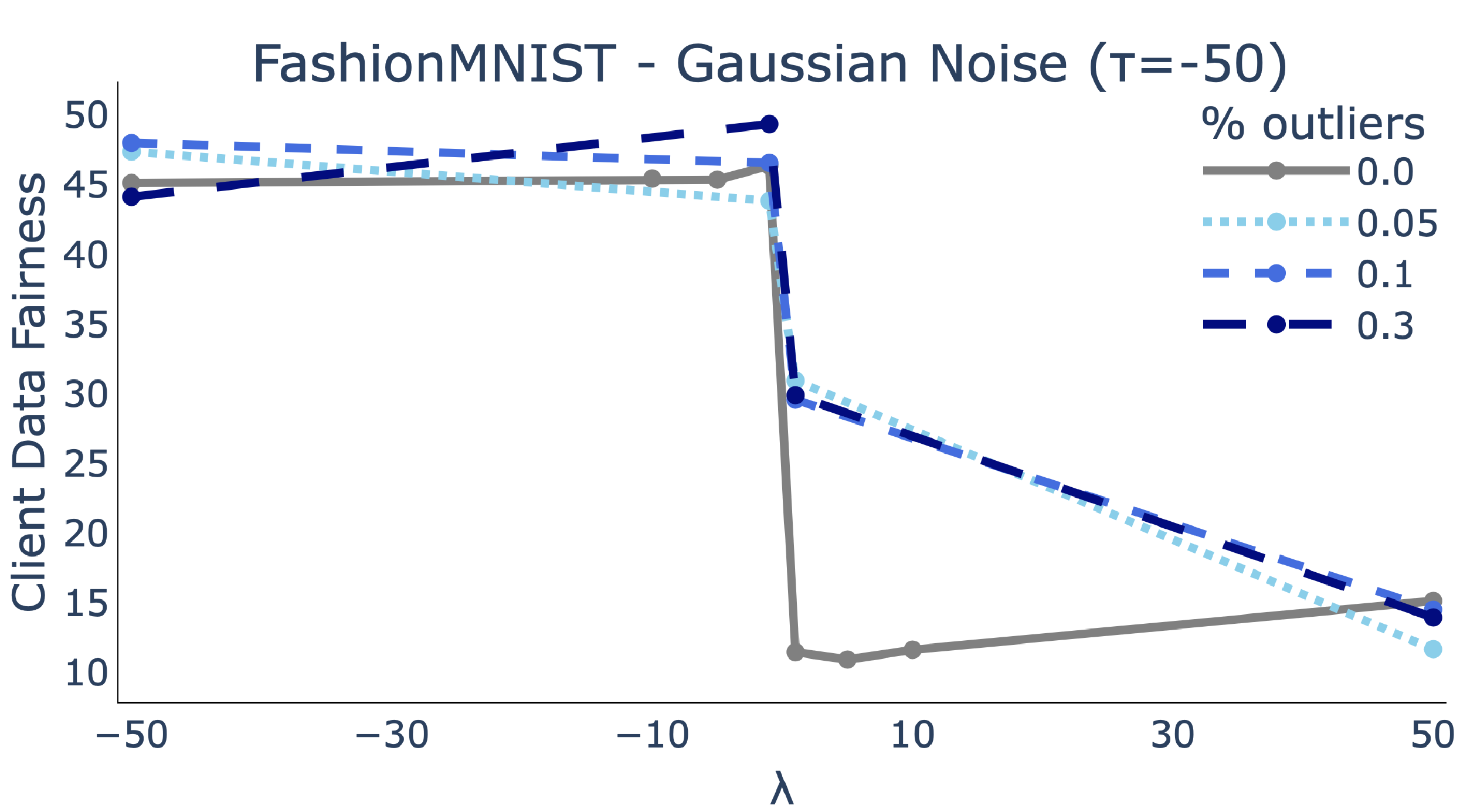}}

    \caption{F-MNIST results---Gaussian noise. Similarly, a larger positive $\lambda=50$ and negative $\tau=-50$ show better two-level fairness and robustness results. 
    } \label{fmnist_gaussian}

\end{figure}

\begin{figure}[!t]
    \centering
    \subfloat[{Test Acc + Client fairness. 
    \label{cifar_Gaussian_cf_tau}}
    ]{\includegraphics[bb=0 0 1200 600, scale=0.098]{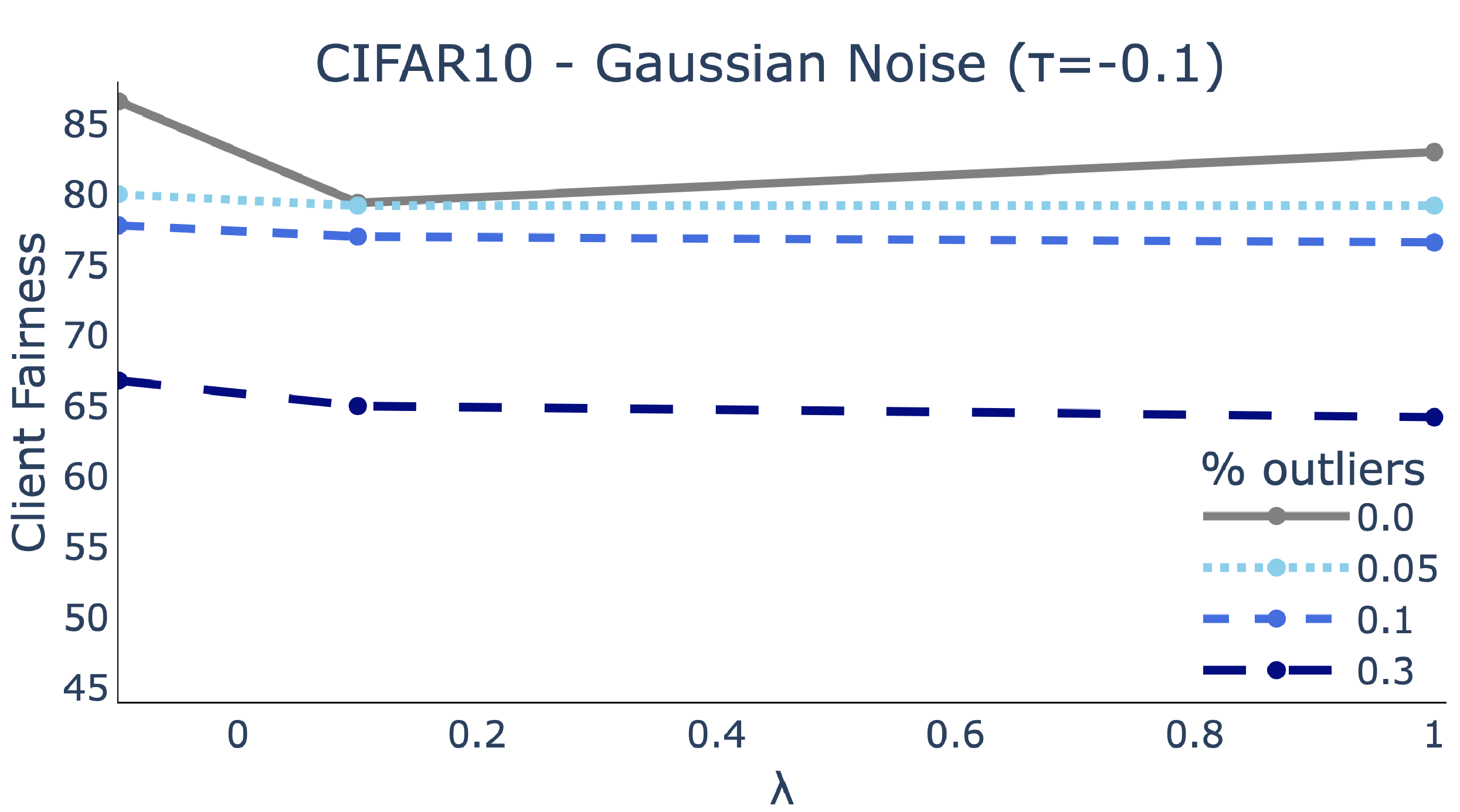}}
    \subfloat[{Client data fairness. 
    \label{cifar_Gaussian_cdf_tau}}
    ]{\includegraphics[bb=0 0 1200 600, scale=0.098]{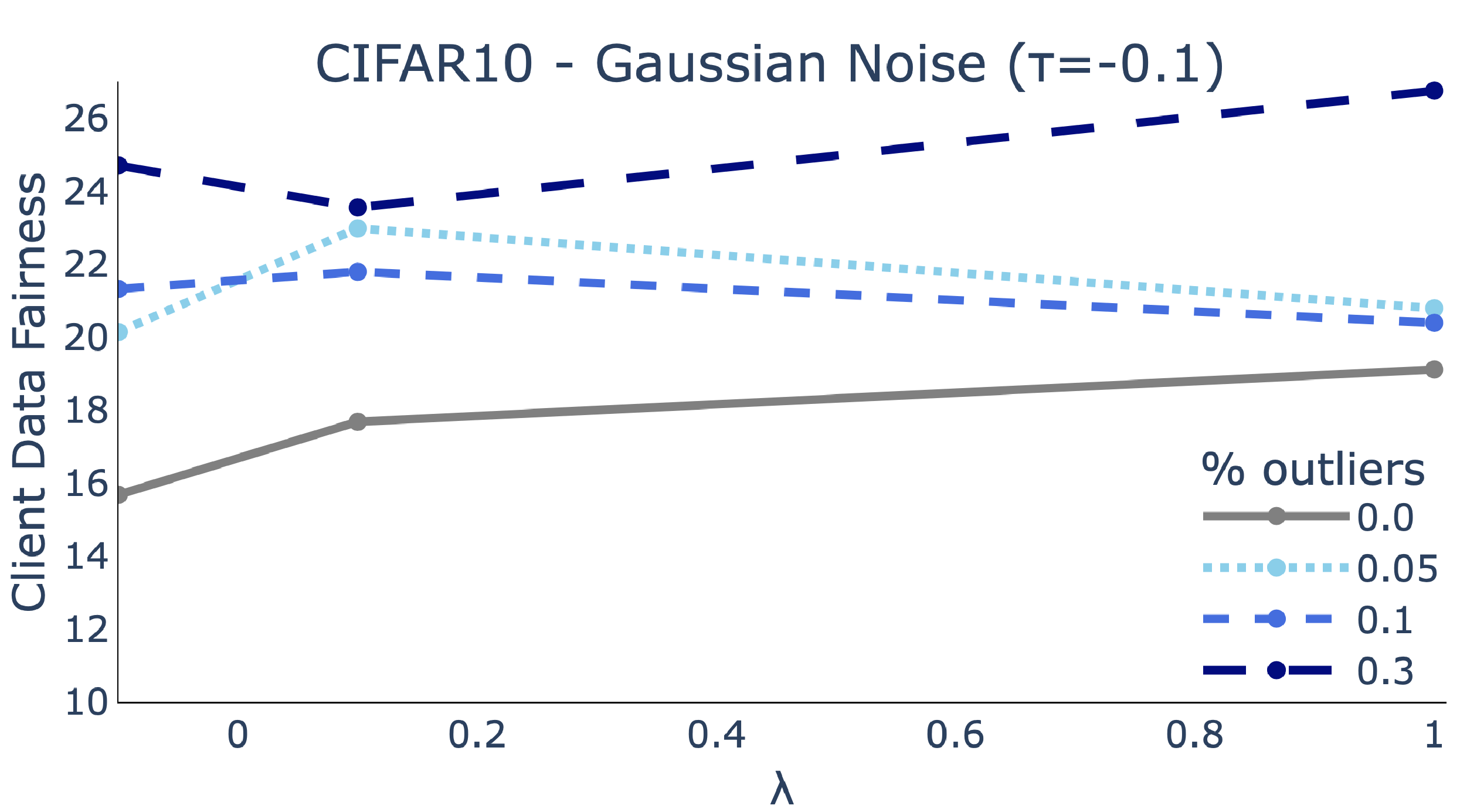}}

    \caption{CIFAR10 results---Gaussian noise. In most cases, better results are obtained with 
    a negative $\lambda$ (e.g., $\lambda=-0.1$). 
    } \label{cifar_Gaussian}
\end{figure}

\subsection{More results}

\begin{table}[!t]
\footnotesize
\centering
\caption{Comparison results -- Clean data} 
\begin{tabular}{llll}
    \hline
    \textbf{MNIST} 
     & Test Acc. & Client fairness & Client data fairness\\
    \hline
    FedAvg & $95.69\%$ & $\sigma=2.91$ & $\mu_\sigma=6.84, \sigma_\sigma=4.90$ \\
    \texttt{Ditto}  & ${\bf 99.25\%}$ & $\sigma=1.27$ & $\mu_\sigma=4.37, \sigma_\sigma=4.23$ \\ 
    FedTilt & $98.53\%$ & $\sigma=1.67$ & ${\bf \mu_\sigma=4.33, \sigma_\sigma=3.33}$ \\
    [0.1cm]\hline
    \textbf{F-MNIST} & Test Acc. & Client fairness & Client data fairness\\
    \hline
    FedAvg & $93.67\%$ & $\sigma=1.97$ & $\mu_\sigma=11.96, \sigma_\sigma=3.52$ \\
    \texttt{Ditto} & $93.77\%$ & $\sigma=5.30$ & $\mu_\sigma=10.89, \sigma_\sigma=7.18$ \\
    FedTilt & ${\bf 96.35}\%$ & ${\bf \sigma=1.85}$ & ${\bf \mu_\sigma=7.61, \sigma_\sigma=3.06}$ \\
    [0.1cm]\hline
    \textbf{CIFAR10} & Test Acc. & Client fairness & Client data fairness\\
    \hline
    FedAvg & $82.20\%$ & $\sigma=4.58$ & $\mu_\sigma=17.96, \sigma_\sigma=3.88$ \\
    \texttt{Ditto} & $74.15\%$ & $\sigma=9.35$ & $\mu_\sigma=18.62, \sigma_\sigma=3.9$ \\ 
    FedTilt & ${\bf 85.24\%}$ & ${\bf \sigma=3.87}$ & ${\bf \mu_\sigma=15.68, \sigma_\sigma=3.69}$ \\
     \hline
\end{tabular} 
\vspace{-4mm}
\end{table}

\begin{table}[!t]
\footnotesize
\centering
\caption{Comparison results -- persistent random corruptions}
\begin{tabular}{llll}
    \hline
    \textbf{MNIST} & Test Acc. & Client fairness & Client data fairness\\
    \hline
    FedAvg & $95.60\%$ & $\sigma=2.86$ & $\mu_\sigma=8.31, {\bf \sigma_\sigma=1.99}$ \\
    \texttt{Ditto}  & ${\bf 98.95\%}$ & $\sigma=1.72$ & $\mu_\sigma=3.86, \sigma_\sigma=5.35$ \\ 
    FedTilt & $98.46\%$ & ${\bf \sigma=1.50}$ & ${\bf \mu_\sigma=2.79}, \sigma_\sigma=3.36$ \\
    [0.1cm]\hline
    \textbf{F-MNIST} & Test Acc. & Client fairness & Client data fairness\\
    \hline
    FedAvg & $95.81\%$ & $\sigma=3.96$ & $\mu_\sigma=10.01, \sigma_\sigma=5.35$ \\
    \texttt{Ditto}  & $34.83\%$ & $\sigma=24.37$ & $\mu_\sigma=21.71, \sigma_\sigma=19.93$ \\ 
    FedTilt & ${\bf 95.96\%}$ & ${\bf \sigma=3.16}$ & ${\bf \mu_\sigma=8.96, \sigma_\sigma=4.55}$ \\
    [0.1cm]\hline
    \textbf{CIFAR10} & Test Acc. & Client fairness & Client data fairness\\
    \hline
    FedAvg & $81.70\%$ & $\sigma=2.27$ & $\mu_\sigma=17.81, \sigma_\sigma=2.94$ \\
    \texttt{Ditto}  & $52.73\%$ & $\sigma=4.71$ & $\mu_\sigma=19.02, \sigma_\sigma=3.20$ \\ 
    FedTilt & ${\bf 82.01\%}$ & ${\bf \sigma=2.17}$ & ${\bf \mu_\sigma=17.36, \sigma_\sigma=2.39}$ \\ 
    \hline
\end{tabular} 
\end{table}

\noindent {\bf Parameter setting.} 
We use a total of 100 clients participating in FL training and assume each client only holds 2 classes to simulate the non-independent identically distributed (non-IID) data across clients in practice. 
The server
randomly selects 10\% clients in each round. 
The used FL algorithms are 
multilayer-perceptron (MLP) for MNIST and F-MNIST, and convolutional neural network (CNN) for CIFAR10. 
By default, we use 10 local epochs and 50 global rounds for MNIST and F-MNIST and 500 rounds for CIFAR10, consider their different convergence speeds. We use SGD to optimize the training with a learning rate 0.01 and mini-batch size 10. We use the Euclidean distance as the default distance function and $\mu=0.01$. 
FedTilt is implemented in PyTorch.  
{Chameleon Cloud} (\url{{https://www.chameleoncloud.org}}) \cite{keahey2020lessons} has served as the platform providing the GPUs to train the FedTilt.
\noindent \textbf{Results on data with Gaussian noises:} Figures~\ref{mnist_gaussian}-\ref{cifar_Gaussian} show the results of  FedTilt on Gaussian noises 
with 
a fixed $\tau$ vs. $\lambda$. 
We see that tuning $\lambda$ can effectively mitigate the effect of outliers. 
Specifically, FedTilt achieves a good two-level fairness with a positive $\lambda$ and is robust to Gaussian noise with the negative $\tau$ on MNIST and FashionMNIST. These results are consistent with the properties of the two-level tilted loss we designed 
---shown in Table~\ref{tbl:tiltsummary}.   
On CIFAR10, the best two-level fairness and robustness tradeoff is obtained with a smaller negative $\lambda=-0.1$---similar to that on the clean data. The injected Gaussian noises possibly increases outliers and we further require a negative $\lambda$ to suppress the effect of the outliers. 
\begin{table}[!t]
\centering
\footnotesize
\caption{Comparison results -- persistent Gaussian noises} 
\begin{tabular}{llll}
    \hline
    \textbf{MNIST} & Test Acc. & Client fairness & Client data fairness\\
    \hline
    \texttt{FedAvg} & $95.41\%$ & $\sigma=3.66$ & $\mu_\sigma=7.36, \sigma_\sigma=5.84$ \\
    \texttt{Ditto}  & ${\bf 98.97\%}$ & $\sigma=1.80$ & ${\bf \mu_\sigma=3.08}, \sigma_\sigma=4.92$ \\ 
    FedTilt & $98.25\%$ & ${\bf \sigma=1.00}$ & $\mu_\sigma=4.39, {\bf \sigma_\sigma=1.67}$ \\
    [0.1cm]\hline
    \textbf{F-MNIST} & Test Acc. & Client fairness & Client data fairness\\
    \hline
    \texttt{FedAvg} & $91.70\%$ & $\sigma=3.51$ & $\mu_\sigma=8.07, \sigma_\sigma=6.14$ \\
    \texttt{Ditto}  & $92.91\%$ & $\sigma=6.82$ & $\mu_\sigma=11.61, \sigma_\sigma=7.50$ \\ 
    FedTilt & ${\bf 94.67\%}$ & ${\bf \sigma=3.37}$ & ${\bf \mu_\sigma=6.92, \sigma_\sigma=2.51}$ \\
    [0.1cm]\hline
    \textbf{CIFAR10} & Test Acc. & Client fairness & Client data fairness\\
    \hline
    \texttt{FedAvg} & $65.61\%$ & $\sigma=6.83$ & $\mu_\sigma=14.09, \sigma_\sigma=6.07$ \\
    \texttt{Ditto}  & $52.43\%$ & $\sigma=12.22$ & $\mu_\sigma=18.45, {\bf \sigma_\sigma=4.64}$ \\ 
    FedTilt & ${\bf 66.80\%}$ & ${\bf \sigma=4.80}$ & ${\bf \mu_\sigma=14.00}, \sigma_\sigma=4.84$ \\
     \hline
    \vspace{-8mm}
\end{tabular} \label{competitiveness_gaussian}
\end{table}

\noindent {\bf Comparing FedTilt with prior works on Gaussian noises:}
In FedTilt, $\tau=50, \lambda=-10$ yield the best results for MNIST and F-MNIST, while $\tau=-0.1,\lambda=-0.1$ remain as the best for CIFAR10. 
Table \ref{competitiveness_gaussian} shows the results. 
We have the below observations: 
1) FedTilt performs the best---most robust to persistent Gaussian noises (i.e., test accuracy is the largest), most fair client performance, and most fair client data performance in the three datasets.  
2) All the compared methods do exhibit robustness to Gaussian noise on MNIST and F-MNIST, but \texttt{Ditto} has a large test accuracy drop on CIFAR10. This indicates the persistent Gaussian noise added to the CIFAR10  data 
can be very harmful for \texttt{Ditto}. One possible reason could be  the injected noisy data prevents \texttt{Ditto} from convergence. Actually, we tested that \texttt{Ditto}'s loss was unstable even with 10,000 global rounds. In contrast, FedTilt converged within 1,000 rounds.   

\end{document}